\documentclass{article}

\usepackage[preprint,final,nonatbib]{neurips_2023}

\usepackage[utf8]{inputenc} 
\usepackage[T1]{fontenc}    
\usepackage{hyperref}       
\usepackage{url}            
\usepackage{booktabs}       
\usepackage{amsfonts}       
\usepackage{nicefrac}       
\usepackage{microtype}      
\usepackage{xcolor}         

\usepackage{amsmath}
\usepackage{amssymb}
\usepackage{mathtools}
\usepackage{amsthm}
\usepackage[shortlabels,inline]{enumitem}
\usepackage{tikz}

\usepackage[capitalize]{cleveref}
\crefname{section}{Sec.}{Secs.}
\Crefname{section}{Section}{Sections}
\Crefname{table}{Table}{Tables}
\crefname{table}{Tab.}{Tabs.}

\newcommand{\RR}{\mathbb{R}}

\theoremstyle{plain}
\newtheorem{theorem}{Theorem}[section]
\newtheorem{proposition}[theorem]{Proposition}
\newtheorem{lemma}[theorem]{Lemma}
\newtheorem{corollary}[theorem]{Corollary}
\theoremstyle{definition}
\newtheorem{definition}[theorem]{Definition}

\theoremstyle{remark}
\newtheorem{remark}[theorem]{Remark}

\usepackage[textsize=tiny]{todonotes}
\usepackage{wrapfig}

\usepackage{xparse}

\NewDocumentCommand{\codeword}{v}{%
\texttt{\textcolor{blue}{#1}}%
}

\newcommand{\cwg}[1]{}
\newcommand{\hjk}[1]{}

\newcommand{\omitt}[1]{}

\makeatletter
\@namedef{ver@everyshi.sty}{}
\makeatother
\usepackage{tikz-cd,extarrows}
\usetikzlibrary{decorations.pathmorphing}

\title{Internal Representations of Vision Models Through the Lens of Frames on Data Manifolds}

\author{%
  Henry Kvinge$^{*\dagger}$, Grayson Jorgenson$^*$, Davis Brown$^*$, Charles Godfrey$^*$, Tegan Emerson$^{*\ddagger\diamond}$ \\
  $^*$Pacific Northwest National Laboratory \\
  $^\dagger$Department of Mathematics, University of Washington\\
  $^\ddagger$Department of Mathematics, Colorado State University,\\
$^\diamond$Department of Mathematical Sciences, University of Texas, El Paso\\
  \texttt{first.last@pnnl.gov} \\
}

\begin{document}

\maketitle

\begin{abstract}
While the last five years have seen considerable progress in understanding the internal representations of deep learning models, many questions remain. This is especially true when trying to understand the impact of model design choices, such as model architecture or training algorithm, on hidden representation geometry and dynamics. In this work we present a new approach to studying such representations inspired by the idea of a frame on the tangent bundle of a manifold. Our construction, which we call a \emph{neural frame}, is formed by assembling a set of vectors representing specific types of perturbations of a data point, for example infinitesimal augmentations, noise perturbations, or perturbations produced by a generative model, and studying how these change as they pass through a network. Using neural frames, we make observations about the way that models process, layer-by-layer, specific modes of variation within a small neighborhood of a datapoint. Our results provide new perspectives on a number of phenomena, such as the manner in which training with augmentation produces model invariance or the proposed trade-off between adversarial training and model generalization.
\end{abstract}

\section{Introduction}


Deep neural networks distill input into semantically meaningful high-level features that lead to robust predictions. Understanding the mechanics of this process this process is important since it is one way of explaining how these models perform at levels that rival human experts. 
Yet, because of the richness of deep learning representations and the high-dimensional spaces they inhabit, existing techniques by necessity provide an incomplete picture of the full relationship between model, training, data, and representation. In particular, we note that while many tools study the large-scale structure of representations (e.g., representation topology \cite{naitzat,rieck2018neural}, mutual information \cite{shwartz2017opening}, model identifiability \cite{roeder2021linear}), there are fewer tools that focus on model behavior in small neighborhoods around a datapoint. Motivated by this, we describe a new tool to illuminate deep learning representations at the local level. To do this we leverage the notion of a frame from differential geometry. A $k$-frame is a choice of $k$ linearly independent vectors from the tangent space of each point $x$ in a $m$-dimensional manifold $M$. The span of these vectors defines a subbundle of the tangent bundle of $M$. By applying the first $\ell$-layers of a deep learning model to a $k$-frame we can see how a model deforms, compresses, or expands specific directions in the immediate neighborhood around a datapoint. We call the resulting structure a {\emph{neural $k$-frame}} (since it will generally not be a true frame). 

Neural frames have several properties that make them a valuable tool which is complementary to other methods of studying neural representations. The first is that through the choice of the input $k$-frame, one can study the ways in which the representation of an input example changes with respect to specific modes of variation. In this work for example, we explore frames that capture the directions of infinitesimal image augmentations, frames that point in the direction of noise, and frames generated by a diffusion model. As we show in Section \ref{sect-experiments}, these different frames are processed in radically different ways by a model even when they all emanate from the same datapoint. There is not an obvious way to do an equivalently targeted exploration for most other methods such as intrinsic dimension or topology. Secondly, a neural frame can be calculated from a single datapoint. This is in contrast to other methods, including all those discussed in \cref{sec:rw}, that focus on large-scale structure and hence require a whole dataset for calculation. Finally, neural frames are an intuitive and flexible construction that can often be integrated into existing tools. For example, in Section \ref{sect-frame-CKA} we show how neural frames can be combined with centered kernel alignment (CKA) to create a method of comparing the local properties of two different representations of a single datapoint.

As a proof of concept, we construct several different flavors of frames for image datasets and then apply them to a range of models with varying architectures and training methods. From these preliminary studies we are able to make a number of observations about small-scale neural representation geometry. (i) Training a model with augmentation causes it to preserve neural frames generated by small augmentations, contradicting intuition that such models would learn to collapse such modes of variation as invariance is learned. (ii) Increasing the $\epsilon$ value used in adversarial training causes a model to increasingly preserve noise directions around a datapoint at the expense more semantically meaningful augmentation directions. (iii) A model's preservation of augmentation directions correlates with its accuracy.

\begin{figure*}[h]
\begin{center}
\includegraphics[width=.88\linewidth]{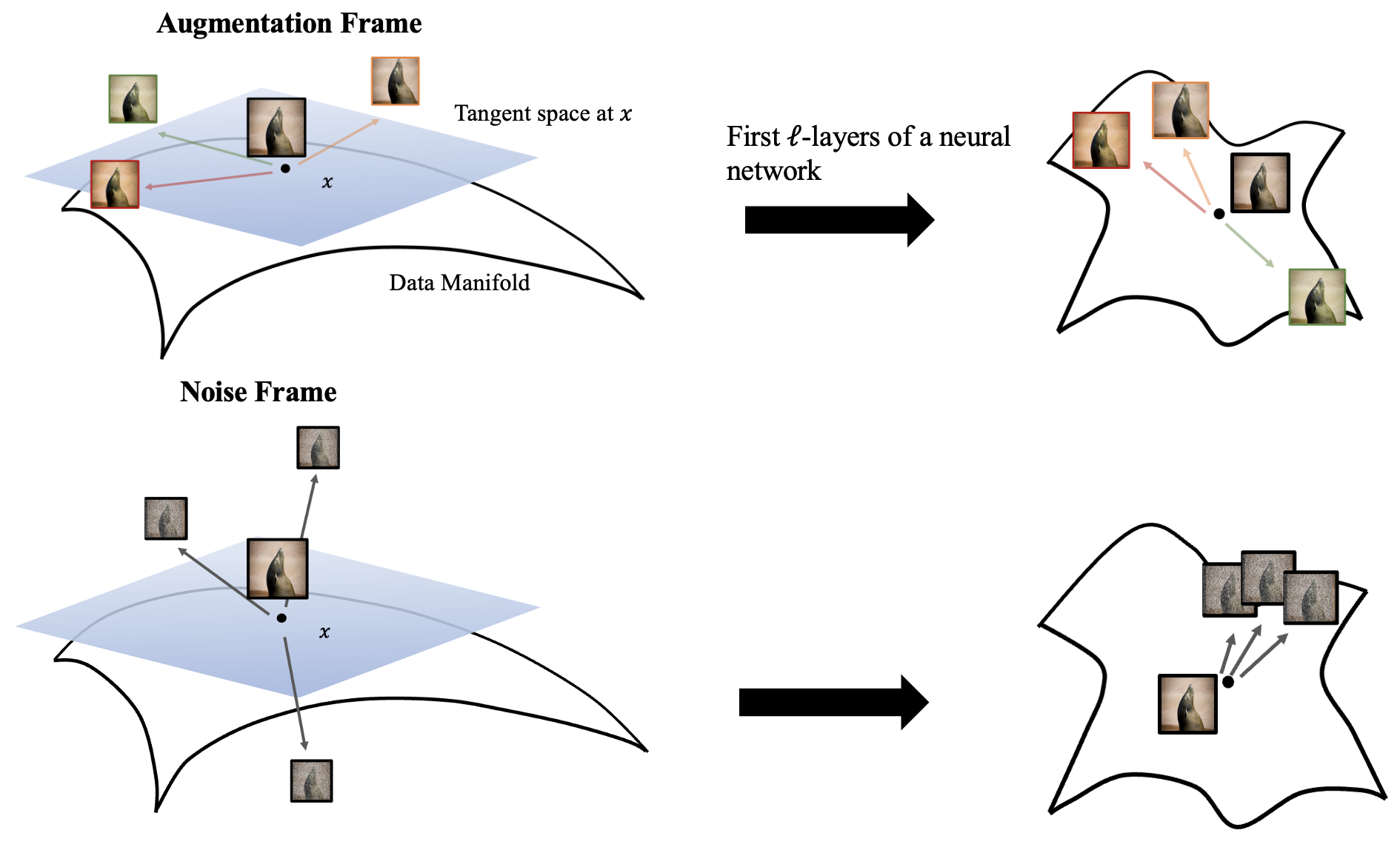}\\
\end{center}
\caption{A cartoon visualizing an augmentation frame with three tangent vectors (derived from hue shift, brightness shift, and a small rotation and crop) and a noise frame with three noise vectors. We think of the  augmentation vectors as \emph{approximately} living in the tangent space of point $x$ on the data manifold while the noise vectors do not. Our cartoon illustrates that models process different frames differently with the noise vectors in this cartoon being collapsed by the first $\ell$ layers of the network. \label{fig-diagram}}
\end{figure*}

In summary, our contributions in this work include the following:
\begin{itemize}\itemsep0em
    \item We describe {\emph{neural frames}}, a flexible tool 
    which can be used to study the small-scale geometry of neural representations. 
    \item We ground the intuitive notion of a neural frame within the theory of frames on a manifold, allowing us to connect our measurements of real models with geometry.
    \item We apply neural frames to a range of models with varying architectures and training histories and show that these are reflected in the local geometry of feature space.
\end{itemize}


\section{Related work}
\label{sec:rw}

\textbf{Mathematical tools to study the internal representations of deep learning models:} 
Understanding the way that neural networks represent input data throughout their layers has been a prominent theme in both theoretical and empirical deep learning research.
While the standard suite of dimensionality reduction methods are used for visualization \cite{colah}, these often fail to be sufficiently quantitative for analysis. Instead, tools from geometry, topology, and probability theory have found application as these fields routinely study spaces that are challenging for humans to understand using visual intuition. Along these lines, several works have explored neural representations in terms of a range of mathematical descriptors, from properties derived from algebraic topology \cite{naitzat,barannikov2021representation}, to intrinsic dimension \cite{ansuini2019intrinsic,pope2021intrinsic,ma2018characterizing}, to mutual information \cite{shwartz2017opening}. 

\textbf{Comparing representations:} Beyond understanding them in isolation, in many cases it is useful to be able to compare the representations of the same data by different models or different layers of the same model. Notable methods include: canonical correlation analysis (CCA) \cite{hardoon2004canonical}, along with its variants SVCCA \cite{raghu2017svcca}
and PWCCA \cite{morcos2018insights} (see \cite{klabunde2023similarity} for a recent survey outlining the relationship between these and other methods). Another method of comparing representations neural stitching \cite{lenc2015understanding} which (roughly) measures whether one model can learn from another's representation. In this work we use centered kernel alignment (CKA) \cite{kornblith2019similarity}, one of the most popular tools for comparing internal representations of deep learning models. 

\textbf{Complementary nature of the tools presented in this paper:} The works discussed above focus on large-scale structure (that is, they consider the set of hidden features of an entire dataset). Our work complements them by focusing on the small-scale structure of representations as well as structure that targets particular modes of variation (e.g., how a model represents noise directions vs directions corresponding to augmentations). 


\section{Neural frames}
\label{sect-neural-frames}

In this Section we introduce neural frames. Where possible, we center our constructions around established notions within manifold geometry since this allows us to prove a number of useful statements. A review of relevant geometric ideas such as the concept of a vector bundle, as well as some of the lemmas that support statements in this section can be found in Section \ref{appendix-geo-background} of the supplementary material.

A {\emph{$k$-frame}} of a finite vector space $V$ of dimension $m \geq k$ is a set of $k$ linearly independent vectors. A {\emph{$k$-frame on $m$-dimensional manifold $M$}} is a choice of $k$-frame for each tangent space $T_xM$ which varies smoothly with respect to the structure of $M$. Given \(k\) vector fields, there is always an open set where they form a \(k\)-frame, whose span is a subbundle of the tangent bundle.

Suppose that we have a data manifold $M$ embedded in ambient space $\mathbb{R}^n$ along with smooth functions $\mathcal{F} = \{f_1, \dots, f_k\}$ with $f_i:(-1,1) \times M \rightarrow M$. If $f_1, \dots, f_k$ satisfy the conditions of \cref{cor-bundle-from-functions}, then we obtain a sub-vector bundle of the tangent bundle of $M$, $V_{\mathcal{F}}$, along with a frame $v_1(x), \dots, v_k(x)$.

Suppose that 
$F = F_\ell \circ \dots \circ F_1: \mathbb{R}^n \rightarrow \mathbb{R}^k$ is a neural network that decomposes into $\ell$ layers $F_i: \mathbb{R}^{n_i} \rightarrow \mathbb{R}^{n_{i+1}}$ with $n_0 = n$ and $n_{\ell} = k$.
We write $F_{\leq i} := F_i \circ \dots \circ F_1: \mathbb{R}^n \rightarrow \mathbb{R}^{n_{i+1}}$ to denote the function that consists of the first $i$ layers of $F$. While there is in general no way to use $F_{\leq i}$ to push forward a vector bundle $V_{\mathcal{F}}$ and hence a $k$-frame, if $F_{\leq i}$ is smooth, then its differential is a linear map  $dF_{\leq i}: T_xM \to T_{F_{\leq i}(x)}\mathbb{R}^{n_{i+1}}$.
This setup allows us to define neural frames.

\begin{definition}
Let $v_1, \dots, v_k$ be a $k$-frame on $m$-dimensional manifold $M$. Let $F$ be a neural network as above. Then $dF_{\leq i}(v_j(x))$ is a tangent vector in $T_{F_{\leq i}(x)}\mathbb{R}^{n_{i+1}}$ and the {\emph{neural $k$-frame at layer $i$ of $F$ and point $x$}} is the set of vectors $dF_{\leq i}(v_1(x)), \dots, dF_{\leq i}(v_k(x))$.
\end{definition}

\textbf{Informally, a neural $k$-frame is simply the object we get when we push a true $k$-frame through the first $i$-layers of a model. }

Even when we have $F_{\leq i}$ and $v_j(x)$, how do we actually compute $dF_{\leq i}(v_j(x))$? This is fortunately simpler than it perhaps looks. Assuming that $v_j(x)$ is derived from a function $f_j: (-1,1)\times M \to M$ as above, fixing \(x \in M \) and letting \(t \in (-1,1)\) vary, the composition \(F_{\leq i} (f(t, x))\) is a smooth path in $\mathbb{R}^{n_{i+1}}$ such that \(F_{\leq i} (f(0, x)) = F_{\leq i}(x)\), and
\begin{equation}
\label{eq:chain}
dF_{\leq i}(v_j(x)) = \frac{\partial F_{\leq i} (f(t, x))}{\partial t}\Big|_{t=0}.
\end{equation}
In practice, we compute $v_i(x) = \frac{\partial}{\partial t}f_i(t,x)|_{t=0}$ by numerically approximating the partial derivative and we compute $dF_{\leq i}(v_j(x))$ by approximating the derivative on the right hand side of \cref{eq:chain}.

Once we have a neural frame, $dF_{\leq i}(v_1(x)), \dots, dF_{\leq i}(v_k(x))$, we can extract a range of statistics that help diagnose what $F_{\leq j}$ is doing to the data manifold at point $x$. 
An initial idea may be to look for changes in rank of the matrices $A_{x,i}$ with columns $dF_{\leq i}(v_1(x)), \dots, dF_{\leq i}(v_k(x))$, as $i$ varies. {\hjk{A range of works have suggested that at some level, the success of deep learning arises from its ability to compress data into representations where a task can effectively solved (e.g., \cite{tishby2015deep, alain2016understanding}). Changes in rank of a frame should, in-principle, capture such compression of the directions of variation represented by the frame, but on a very local level not investigated in other works}}
Unfortunately, rank is sensitive to noise and is hence unsuitable for this application.\footnote{By ``rank is sensitive to noise" we mean the following: given vectors \(v_1, \dots, v_k \in \RR^n \) and e.g. standard normal random variables \(Z_1,\dots, Z_n \), and any \(\epsilon >0\),  \(\mathrm{rank}(v_1+\epsilon Z_1,\dots, v_k + \epsilon Z_k) = \min(k, n)\) with probability 1 (\emph{regardless} of  \(\mathrm{rank}(v_1,\dots,v_k)\)). On the other hand one can check that \(r(v_1+\epsilon Z_1,\dots, v_k + \epsilon Z_k) \to r(v_1,\dots,v_k)\) as \(\epsilon \to 0 \).} An appealing alternative is \emph{stable rank} \cite{rudelson2007sampling}, which is the ratio between squared Frobenius norm and the squared spectral norm of a matrix. For $A_{x,i}$ this is
\begin{equation*}
    r(A) := \frac{||A_{x,i}||_{\mathrm{frob}}^2}{||A_{x,i}||_{\mathrm{spec}}^2}.
\end{equation*}
It can easily be calculated by computing (in terms of the sorted singular values $\sigma_1 \geq \sigma_2 \geq \dots \geq \sigma_k$ of $A_{x,i}$) \(r(A) = \frac{1}{\sigma_1^2} \sum_i \sigma_i^2\).
Since $k$ is small in practice ($< 50$), computing stable rank via a singular value decomposition of $A_{x,i}$ is quick. We note that stable rank has found a number of useful applications to deep learning \cite{sanyal2019stable}.

Stable rank is a lower bound on rank, and we can interpret a decrease in stable rank when moving from $v_1(x),\dots,v_k(x)$ to $dF_{\leq i}(v_1(x)), \dots, dF_{\leq i}(v_1(x))$ to mean that $F_{\leq j}$ compresses the data manifold $M$ in the directions captured by $v_1(x),\dots,v_k(x)$. This of course does not mean that $F_{\leq j}$ compresses all of $M$ at $x$ since in general $v_1(x),\dots,v_k(x)$ will only span a subspace of the tangent space of $M$ at $x$. Nevertheless we will see that some of results we obtain using stable rank reflect patterns seen in past intrinsic dimension experiments. It is worth mentioning that while true linear algebraic rank has the property that the rank of a product \(AB\) is at most the minimum of the ranks of \(A\) and \(B\), this fails for stable rank. However, the extent of this failure is controlled by the behavior of the top singular values of \(A, B\) and \(AB\) --- precisely:
\begin{lemma}
\label{lem:stable-rank-noninc}
If \(A \) and \(B\) are \(l \times m \) and \(m \times n\) matrices respectively, then 
\begin{equation}
\label{eq:stable-rank-noninc}
    r(AB) \leq \Big(\frac{\lVert A \rVert_{\mathrm{spec}} \lVert B \rVert_{\mathrm{spec}}}{\lVert AB \rVert_{\mathrm{spec}}}\Big)^2 \min \{r(A), r(B) \}.
\end{equation}
\end{lemma}
We include Lemma \ref{lem:stable-rank-noninc} since when taking \(A = dF_{i+1} \) and \(B = dF_{\leq i} \), so that \(AB = dF_{\leq i+1}\), Lemma \ref{lem:stable-rank-noninc} seems somewhat explanatory of the general decreasing trend in the curves of \cref{fig-stable-rank-robust-resnet50-aug,fig-stable-rank-augmentation,fig-stable-rank-different-frames-resnet50,fig-ID-vs-SR-vit}. \footnote{Note that by the multiplicative property of the spectral norm, the first factor on the right hand side of \cref{eq:stable-rank-noninc} is always \(\geq 1 \).} 

We end this section by discussing two different types of frames which seem particularly interesting. Methods of finding optimal frames are likely to depend on the information that one is interested in extracting from a model and are an interesting direction of research. \\

\vspace{1mm}
\noindent\textbf{Augmentation frames:} This neural frame is generated by image augmentations $f_i: (a,c) \times M \rightarrow M$ with the following properties: (1) as implied by the domain and range of $f_i$ above, $f_i$ transforms one natural image (on $M$) to another natural image (on $M$). While this latter image was not actually captured by a camera (instead being produced in software) it should be plausible that it could have been. (2) Aside from the input image, $f_i$ is also controlled by a parameter $t \in (a,c)$ for $a < c \in \mathbb{R}$ such that for some $b \in (a,c)$, $f_i(b,x) = x$ for all $x \in M$. For example, if $f_{\text{rot}}: (-180,180) \times M \rightarrow M$ is image rotation, with the first parameter measuring the number of degrees that an image will be rotated, then it is always the case that $f_{\text{rot}}(0,x) = x$. Note that such functions can always be reparametrized to fit the form in \cref{lem:flows} and \cref{cor-bundle-from-functions}.

The frame $v_1,\dots,v_k$ derived from $f_1, \dots, f_k$ describes a number of pseudo-naturalistic directions in which an image can vary without leaving the image manifold. The neural frame associated with this frame tells us how a model handles change in these directions locally. In \cref{tab:augmentation-transformations}, we list the image augmentations that we used, the library we used to implement them, and the augmentation parameters that were used in our experiments.

\begin{table*}[thb]
\caption{Augmentation transformations and the library and parameters used to implement them.}
\label{tab:augmentation-transformations}
\begin{center}
\begin{tabular}{rcl}
\toprule
\textbf{Augmentation} & \textbf{Library} & \textbf{Default parameters}
 \\ \midrule
JPEG transform & imgaug library \cite{imgaug} & Compression $70\%$ \\
Brightness &   Torchvision & Brightness factor $1.02$   \\
Crop and resize & Torchvision & Resize $\times 4$, crop off $1$ pixel,\\
& & Interpolation: bilinear, nearest, linear\\
Contrast &  Torchvison  & Contrast factor $1.05$ \\
Gamma transform & Torchvision & Gamma $1.02$ \\
Hue & Torchvision & Hue factor $.01$\\
Saturation & Torchvision & Saturation factor $1.1$\\
Sharpness & Torchvision & Sharpness factor $1.2$\\
Downscale & Torchvsion & Resize $\times 0.9$, return to original\\
& & Interpolation: bilinear, nearest, linear\\
Rotation + translation & Torchvision & Resize $\times 4$, rotate $2$ degrees, \\
& & centers $(0,0),(50,50),(-50,50)$\\
Gaussian blur & Torchvision & Kernel size $3 \times 3$, $\sigma = 2.0$\\
Log correction & Kornia \cite{eriba2019kornia} & Gain $1.05$ \\
Sigmoid transform & Kornia & Cutoff $0.5$, gain $5$ \\
\end{tabular}
\end{center}
\end{table*}

Some image augmentations come with more than a single real parameter that a user can choose from. For example, when rotating an image, one can often pick the pixel coordinates of the point which will be the folcrum of the rotation (for example, in \cite{marcel2010torchvision}). How many versions of the augmentation should one add to the augmentation frame in such cases? In a $224 \times 224$ image we have $50176$ pixels that we could rotate around. How many can be added before the corresponding tangent vectors become linearly independent? In cases where the underlying augmentation corresponds to the action of a Lie group (including this case, where the Lie group is the special Euclidean group $SE(2)$), Lie theory can provide an answer. We begin by recalling that the action of a Lie group $G$ on a manifold $M$ induces a linear map from the Lie algebra $\mathfrak{g}$ to $T_xM$ for any $x \in M$.

\begin{proposition} {(Theorem 20.15 \cite{lee2013smooth})}
\label{prop-lie-algebra}
Let \(\mathfrak{g} = T_0 G\) be the Lie algebra of \(G\) and $\rho: G \times M \rightarrow M$ a Lie group action of $G$ on $M$. Suppose \(x \in M \), and define \(\mathrm{ev}_x: G \to M \) as \(\mathrm{ev}_x(g) = \rho(g, x) \). Then $\mathrm{ev}_x$ is a smooth map and the differential of \(\mathrm{ev}_x \) at the identity element \(e \in G\) is a linear map \(d \mathrm{ev}_x: \mathfrak{g} \to T_x M\).
\end{proposition}

Given Proposition \ref{prop-lie-algebra}, our problem is equivalent to identifying the dimension of the image of \(d \mathrm{ev}_x\) 
This will tell us the maximum number of linearly independent tangent vectors that can be generated by the action of $G$. 
To state the solution, we require a piece of terminology: the \emph{stabilizer} of a point \(x \in M\) is the subgroup \( G_x = \{g \in G \, | \, gx =x \}\).
\begin{proposition}
\label{prop:orbit-stab}
The natural map \(\mathfrak{g} \to T_x M\) is injective if and only if the stabilizer \(G_x \) is discrete.
\end{proposition}

In our rotation example 
the stabilizer $SE(2)_x$ of a natural image \(x\) almost always consists of the identity alone, hence is in particular discrete.
Since the dimension of Lie group $SE(2)$ is 3 and a Lie algebra's dimension (as a vector space) is equal to the manifold dimension of its corresponding Lie group, the subspace of $T_xM$ spanned by tangent vectors generated by all possible rotations at different points in a image is 3. Thus we conclude that for most images we only need to include tangent vector approximations for rotations at 3 points in an image. In our experiments we choose to rotate at pixels $(0,0), (50,50),$ and $(-50,50)$. \cwg{There could be a mention here 
of the sanity test you ran Henry.}
\\
\vspace{.1mm}

\noindent\textbf{Diffusion frames:} In the recent work \cite{luzi2022boomerang}, Luzi et. al. describe how diffusion models can be used to sample locally around an image. In essence, the method they describe, called Boomerang, adds a user chosen amount of noise to an image (driving it toward the latent space of the model) and then uses the diffusion model to bring it back to the space of natural images. In this process the image will be subtly altered in a naturalistic way. This method fits nicely within our scheme of neural frames: we use Boomerang to produce $k$ distinct perturbations of an image (to match our augmentation frame, in our experiments $k = 19$), then we define a $k$-frame with these. We assume that the perturbations generated by the diffusion model are small enough so that the linear path from a perturbed image to the real image lies on the image manifold.

\section{Experiments}
\label{sect-experiments}

Having developed neural frames, we show their utility by using them to probe the local behavior of deep learning models. 
We give full experimental details in Section \ref{appendix-experimental-details} in the supplementary materials. 
Unless noted otherwise, we use publicly available weights from torchvision \cite{marcel2010torchvision} or timm \cite{rw2019timm}. 
To simplify diagrams, we omit layer names providing their numerical correspondence in Tables \ref{table-vit-layers}-\ref{table-alexnet-layers} in the supplementary material. 

We utilize four different types of frames in our experiments which we describe here. (1) \emph{Gaussian noise:} We perturb an image with random Gaussian noise with mean and variance which we normalize to match the statistics of vectors in our augmentation frame. Note this is not a frame on the image manifold itself. (2) \emph{Augmentation frame:} We use the augmentations listed in Table \ref{tab:augmentation-transformations} to generate an augmentation frame (example images of this frame are found in Figure \ref{fig-example-augmentations} in the supplementary material). 
(3) \emph{Random rotation of augmentation frame:} We randomly rotate the augmentation frame above so it retains its geometric structure but loses its semantic meaning.
(4) \emph{Perturbations via stable diffusion:} We use the Boomerang method \cite{luzi2022boomerang} to generate samples from around an ImageNet image and take these samples as perturbations to build a frame (example images of this frame are found in Figure \ref{fig-example-sd}).


\begin{figure}[h]
\begin{center}
\includegraphics[width=.6\linewidth]{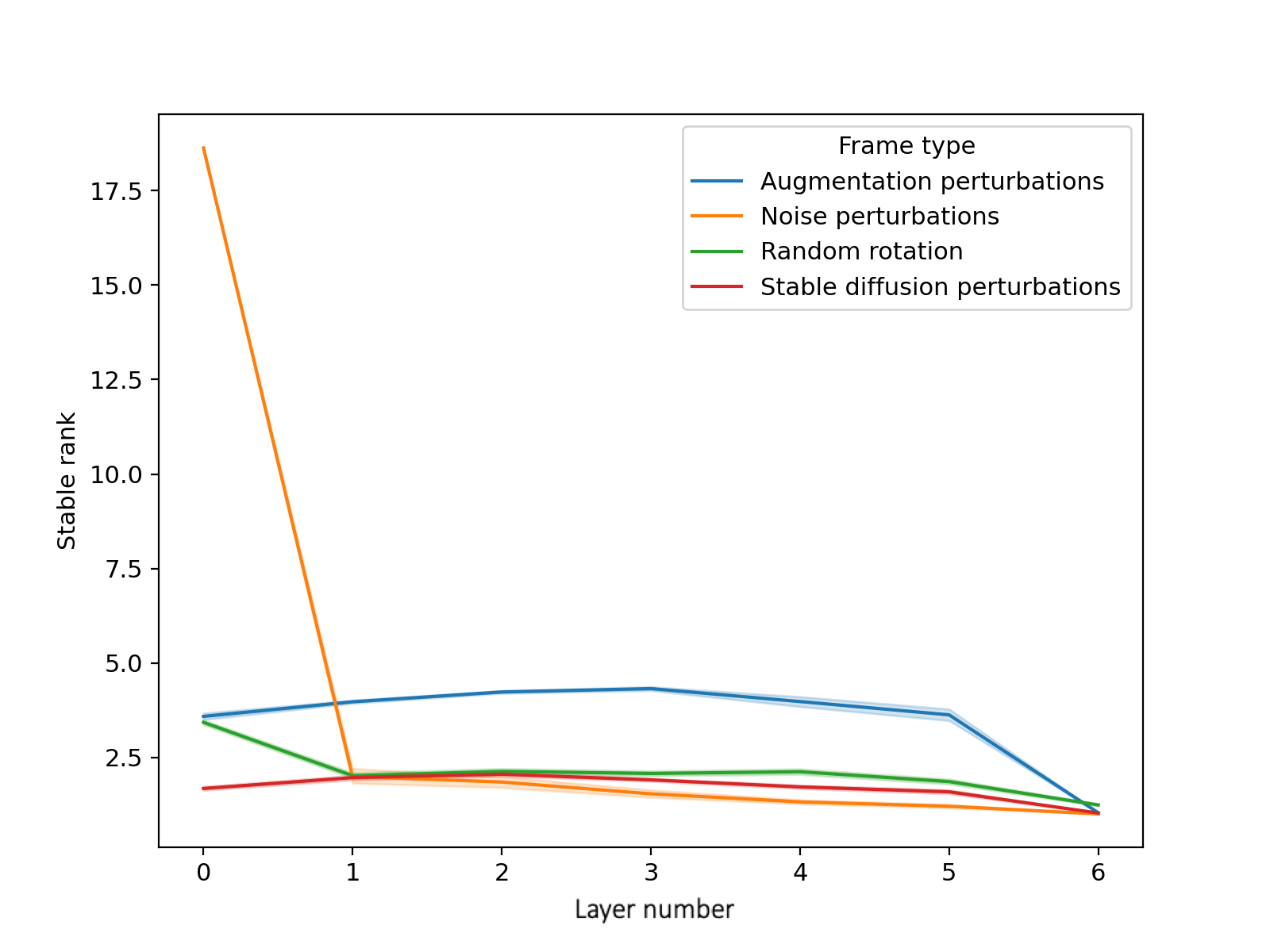}
\end{center}
\caption{The stable rank of different types of frames measured at various layers of a ResNet50 model. Layer zero corresponds to model input and Layer 6 corresponds to model output. Shaded regions indicate $95\%$ confidence intervals over $40$ randomly selected ImageNet images. \label{fig-stable-rank-different-frames-resnet50}}
\end{figure}

\noindent\textbf{In general, models preserve on-manifold tangent vectors and collapse vectors that point off manifold:} It is reasonable to ask whether a CNN or transformer actually ``sees'' different frames differently. One might worry that on the small scales that we work, the different frames we construct do not capture meaningful differences in model representations. For example, can a model actually tell the difference between an augmentation frame and a noise frame? 

To test this, we plot the stable rank for a number of different frames at different layers of a ResNet50 pretrained on ImageNet \cite{marcel2010torchvision}. The results are shown in Figure \ref{fig-stable-rank-different-frames-resnet50} (Figure \ref{fig-stable-rank-different-frames-vit} in the supplementary material for the same experiment with a ViT). We can see that even with the coarse statistic of stable rank, different neural frames exhibit distinct behavior when processed by the models. The neural frame generated from Gaussian noise predictably has the highest stable rank in the ambient space (layer 0), but this drops quickly in both models as the frame is processed. This suggests that as we might intuitively expect, models preserve those directions more representative of natural variation at the expense of random directions. In contrast, augmentation frames, which simulate directions of natural variation of imagery are generally more preserved from layer to layer, only being collapsed in the final classification layer. 

One might wonder if this phenomenon is a consequence not of the directions that augmentation frames point relative to noise frame, but of other structural features of the frame itself. For example, inspection of the stable rank of the input frames in Layer 0 of Figure \ref{fig-stable-rank-different-frames-resnet50} (prior to processing by the model) show that the noise frame is close to being an orthogonal set of vectors while the augmentation frame has significant linear redundancies. We test this by randomly rotating the augmentation frame so it no longer points in the direction of natural changes to the image but keeps other structural features. When we do this we see that this rotated frame, like the noise frame, is collapsed by the model. This provides strong evidence that, even at the very smallest scales, models recognize and preserve directions that simulate natural variation found in imagery.


\begin{figure}[h]
\begin{center}
\includegraphics[width=.6\linewidth]{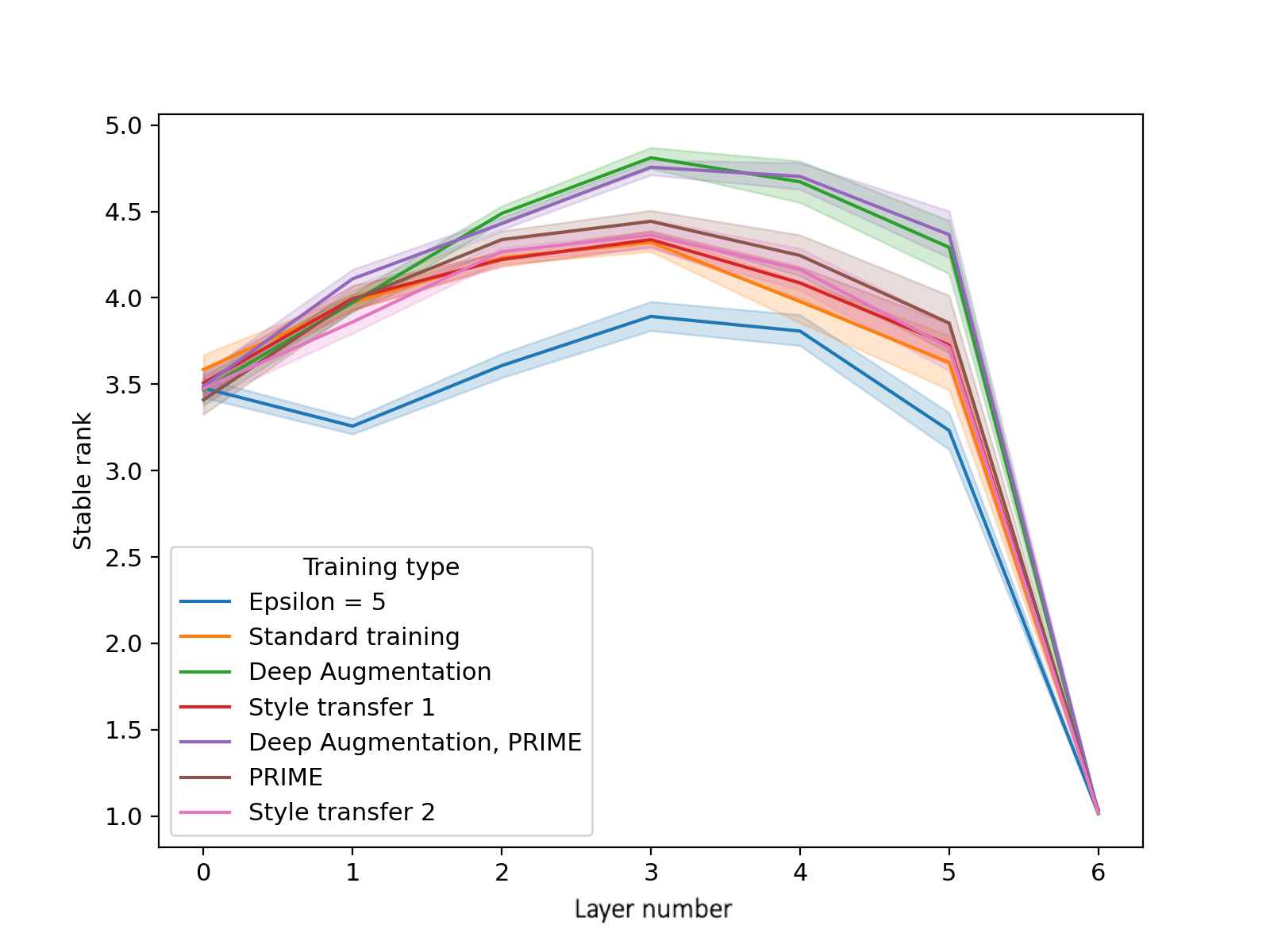}
\end{center}
\caption{The stable rank (by layer) of ResNet50 models evaluated with respect to augmentation frames on ImageNet images where each model was trained with a different augmentation method. Layer zero corresponds to model input and the last layer corresponds to model output. Shaded regions indicate $95\%$ confidence intervals over $40$ randomly selected ImageNet images. \label{fig-stable-rank-augmentation}}
\end{figure}

\noindent\textbf{Training with augmentation causes models to better preserve on-manifold augmentation frames:} In order to better understand the impact of training with augmentation at the local level, we explored the stable rank of augmentation frames for a range of models trained with (and without) different types of heavy augmentation. We consider ResNet50 models trained with the augmentation methods PRIME \cite{modas2021prime}, 
Deep Augmentation \cite{hendrycks2021many}, 
and Stylized ImageNet \cite{geirhos2018imagenet}.  Our results are shown in Figure \ref{fig-stable-rank-augmentation} where we see that generally, models trained with extra augmentation have neural frames with higher stable rank. This indicates that these models more faithfully preserve (and to some extent even expand) frames represented by small augmentations. Note that this may seem unexpected given that training with augmentation is generally done to build in invariance natural variation in images. This might lead one to conclude that training with augmentation should cause augmentation frames to collapse as a model consolidates different augmented versions of the same data point. Figure \ref{fig-stable-rank-augmentation} suggest that this must happen only in the final layers of a model and that instead, training with augmentation causes a model to learn more distinct and structured representations of different augmentations of a single input, collapsing these to achieve invariance only at the final layer. This speculation agrees with observations found in \cite{kvinge2022ways}. 


\noindent\textbf{On the other hand, adversarial training degrades the preservation of augmentation frames but improves the preservation of noise frames:}
It has been empirically confirmed via a range of different methods that adversarial training has effects on the way that computer vision models process data at the local level \cite{engstrom2019adversarial}. To investigate whether this can be seen at the level of neural frames, we calculated the stable rank for 5 layers of several different ResNet50 models, each trained with a different $l_2$-robust $\epsilon$ bound of adversarial training with weights from \cite{salman2020adversarially}. 

On the left in Figure \ref{fig-stable-rank-robust-resnet50-aug} we show the stable rank over $40$ random ImageNet images with respect to the augmentation frame and models with various strengths of adversarial training (here $\epsilon$ gives the $\ell_2$ bound on adversarial examples shown to the model during training). We observe that the stable rank of our augmentation frames generally decreases slightly as the strength of adversarial training increases. Furthermore, these differences are most pronounced at earlier layers of the model. On the other hand, we can see that when we substitute the augmentation neural frame for the off-manifold noise neural frame (right, Figure \ref{fig-stable-rank-robust-resnet50-aug}) that the opposite pattern holds and stable rank generally increases as the strength of adversarial training increases. 

\begin{figure*}[h]
\begin{center}
\includegraphics[width=.49\linewidth]{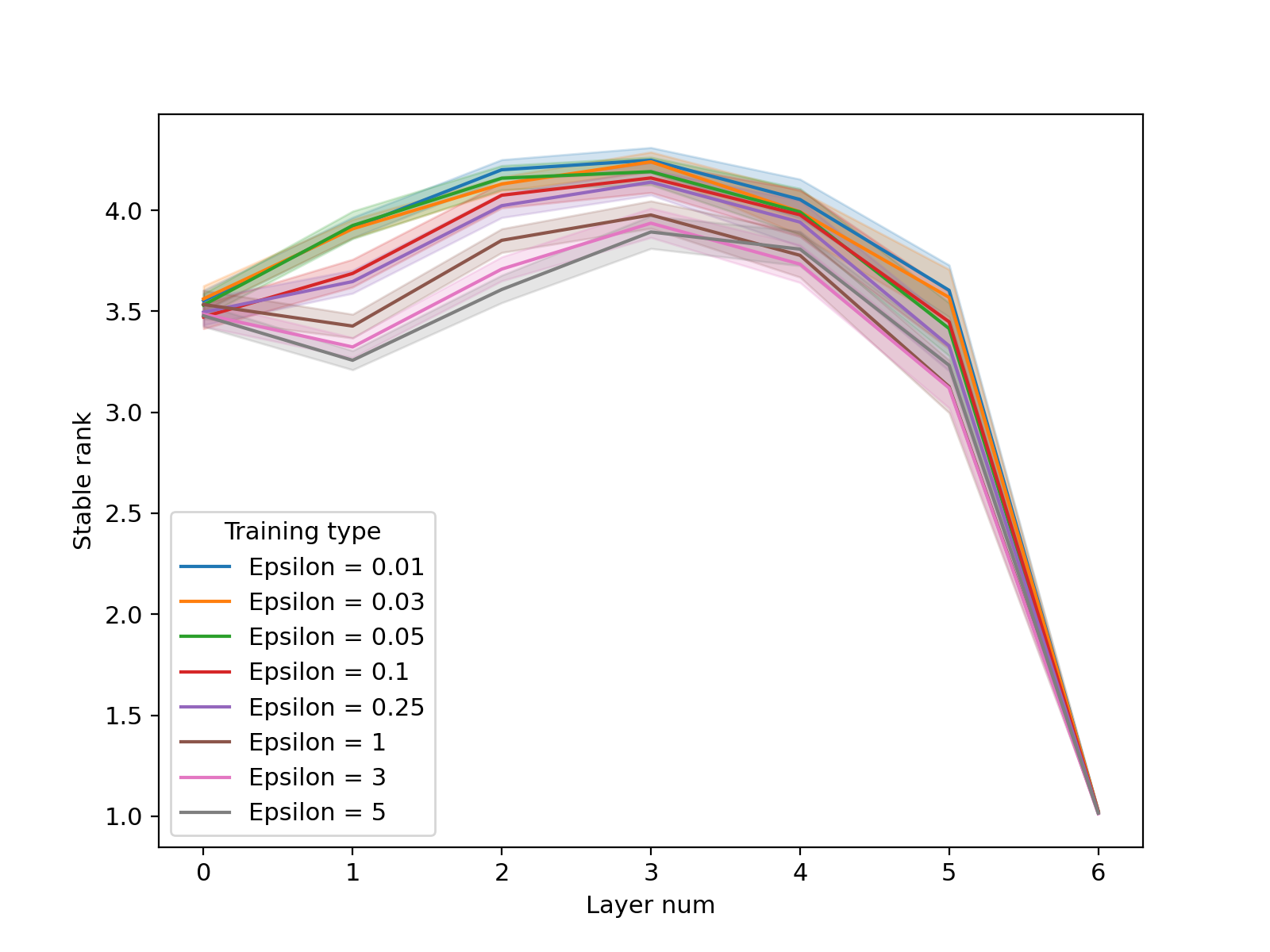}
\includegraphics[width=.49\linewidth]{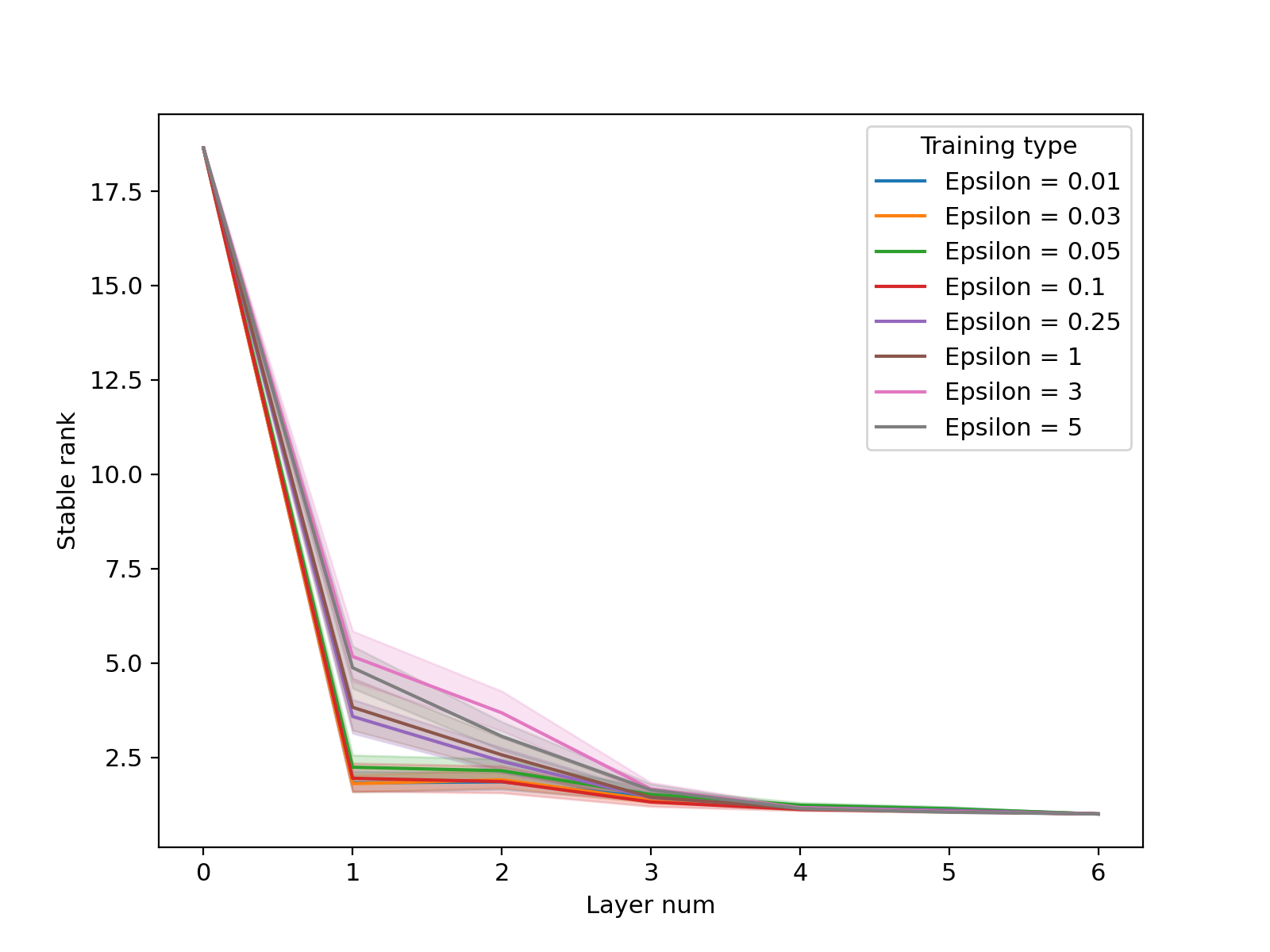}
\end{center}
\caption{{\textbf{(Left)}} The stable rank (by layer) of adversarially trained $l_2$-robust ResNet50 models with varying $\epsilon$ values evaluated with respect to augmentation frames on ImageNet images. Layer zero corresponds to model input and the last layer corresponds to model output. {\textbf{(Right)}} The same models evaluated on noise frames. Shaded regions indicate $95\%$ confidence intervals over $40$ randomly selected ImageNet images. \label{fig-stable-rank-robust-resnet50-aug}}
\end{figure*}

\noindent \textbf{The stable of rank augmentation frames is correlated with model accuracy:} It is natural to ask whether statistics associated with neural frames have any relationship with other characteristic of a model. In Figure \ref{fig-stable-rank-vs-acc}, we show that generally, higher stable rank of augmentation frames is correlated with model accuracy. This observation fits well with our speculation above that preservation of augmentation frames (as measured by stable rank) may be tied to a model's fit to the underlying image manifold.

\begin{figure}
\begin{center}
\includegraphics[width=.7\linewidth]{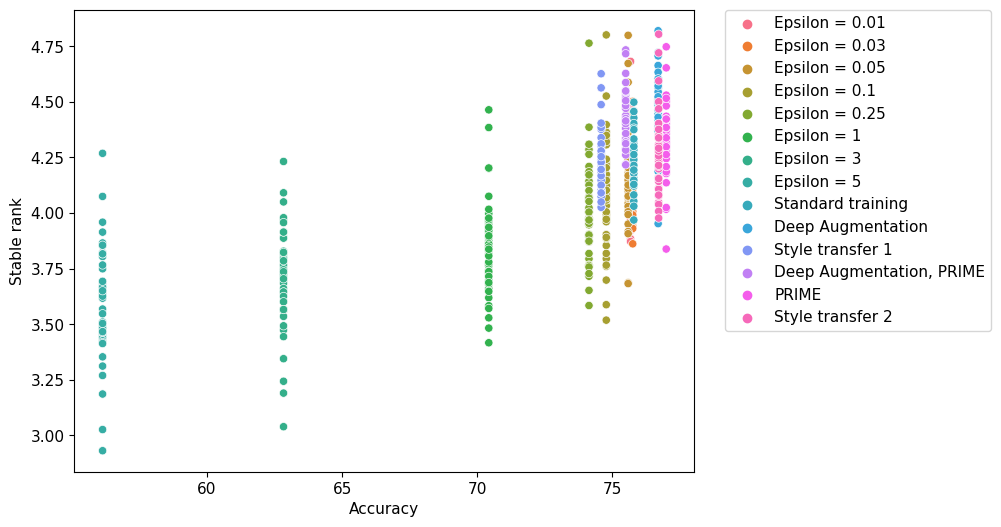}
\end{center}
\caption{The stable rank of different types of frames measured at various layers of a ResNet50 model. Layer zero corresponds to model input and Layer 6 corresponds to model output. Shaded regions indicate $95\%$ confidence intervals over $40$ randomly selected ImageNet images. \label{fig-stable-rank-vs-acc}}
\end{figure}


We end with a couple final observations that we explore more thoroughly in the supplementary material: (i) Neural frames reveal that over the course of training, models initially locally compress the image manifold and then gradually expand it as training progresses (see Figure \ref{fig-stable-training} in the supplementary material). It would be interesting to tie this to previously observed phenomena e.g., the information bottleneck \cite{tishby2015deep}. (ii) The stable rank of augmentation neural frames tend to vary less across CNN models when compared to transformers, suggesting that transformers compress and stretch data manifolds more during processing (see Figures \ref{fig-stable-rank-architectures-CNN}-\ref{fig-stable-rank-architectures-ViT}). (iii) We describe frame CKA, which can reveal inter-layer differences in ViTs that traditional CKA struggles to capture.

\section{Limitations}

While neural frames provide unique and valuable information about how a network processes the tangent bundle of a data manifold, this information is never the full story. For example, in most cases the frames we use span proper subspaces of the tangent space. Thus, there may be changes to the tangent space that we miss because they are orthogonal or nearly orthogonal to all vectors in the frame that we use. Augmentation frames may be challenging to use in specialized scenarios where augmentations that preserve the particular image manifold under consideration are not a priori known. 

\section{Conclusion}

While data manifolds play a central role in our understanding of how and why deep learning works, extracting any tangible information about them is challenging. In this paper we provide a new tool, neural frames, to help probe the ways that deep learning models interact with data manifolds. We show that neural frames capture some of the information that can be obtained by other more data-hungry approaches, but also provide their own set of unique insights. 

\begin{ack}
This research was supported by the Mathematics for Artificial Reasoning in Science (MARS) initiative via the Laboratory Directed Research and Development (LDRD) investments at Pacific Northwest National Laboratory (PNNL). PNNL is a multi-program national laboratory operated for the U.S. Department of Energy (DOE) by Battelle Memorial Institute under Contract No. DE-AC05-76RL0-1830.
\end{ack}

\bibliographystyle{plain}
\bibliography{neurips}

\clearpage
\newpage
\appendix

\section{How does the stable rank of a frame relate to intrinsic dimension?}

Given that the dimension of a manifold can be defined as the vector space dimension of its tangent space, one might ask how the stable rank of a neural frame (which in some cases is also related to the tangent space of a data manifold) relates to the intrinsic dimension of a neural representation. This is especially pertinent given the large number of works that investigate neural representations through the lens of intrinsic dimension \cite{ansuini2019intrinsic,pope2021intrinsic,amsaleg2017vulnerability,ma2018characterizing}. In this short section we compare the stable rank of neural frames to intrinsic dimension to get a better sense of what both are capable of telling us.

\begin{itemize}[leftmargin=*]
\item \textbf{Manifold dimension:} Intrinsic dimension is designed to estimate the dimension 
 of the manifold underlying a dataset. Unless we are using a frame (whose vectors span the entire tangent space of the manifold), the stable rank will not tell us the intrinsic dimension (though it is a lower bound).\\
\item \textbf{Number of points required:} Intrinsic dimension generally requires many real data points to calculate, and this number increases as the intrinsic dimension increases. A broad range of work has tried to provide detailed estimates of the number of points necessary to get a reliable estimate (e.g., \cite{fefferman2016testing}). On the other hand, the stable rank of a frame can be calculated with a single datapoint provided one knows how to perturb that datapoint in order to generate the frame. Because there is variation between the stable rank of frames on different datapoints, we advocate using at least several datapoints and taking an average.\\
\item \textbf{Analyzing different sources of variation:} A variety of works that have investigated the intrinsic dimension of the hidden activations of deep learning models have noted that these models tend to decrease the original dimension of the data manifold as data passes through the model \cite{ansuini2019intrinsic}. One can ask what specific sources of variation are collapsed in this process. Does a drop in intrinsic dimension from one layer to the next represent the fact that the model ignores some structured degree of freedom (e.g., color)? It is not straightforward to measure this with intrinsic dimension alone. On the other hand, since frames can capture specific directions of variation, these kinds of questions become accessible when using this tool.\\
\item \textbf{Local vs very local:} Intrinsic dimension estimators come in a variety of flavors. Some use the entire dataset to estimate dimensionality, while other more recent approaches average over many local neighborhoods. In all these cases, the size of the neighborhood used is constrained by the dataset. If the dataset is sparsely sampled (and hence points are further apart), the neighborhoods used are by necessity larger. On the other hand, since neural frames utilize various tools to perturb a datapoint, the neighborhoods they study can often be made much smaller.\\
\item \textbf{Use of real vs synthetic data:} In most cases intrinsic dimension estimators only use real datapoints from the dataset. On the other hand, the approaches to neural frames that we describe here use augmentations or generative models to create close neighboring points that estimate tangent vectors.
\end{itemize}

Despite these differences, we find that in certain cases at least, intrinsic dimension and the stable rank of neural frames appear to capture similar patterns in neural representations. Figure \ref{fig-ID-vs-SR-vit} shows the intrinsic dimension of 5,000 ImageNet images at different layers of a vision transformer (left vertical axis) as captured by intrinsic dimension estimators MLE \cite{mle} and TwoNN \cite{facco2017estimating}, vs the stable rank of an augmentation frame (right vertical axis). We see that while these statistics differ numerically, their curves have similar shapes. 

\begin{figure}[h]
\begin{center}
\includegraphics[width=.75\linewidth]{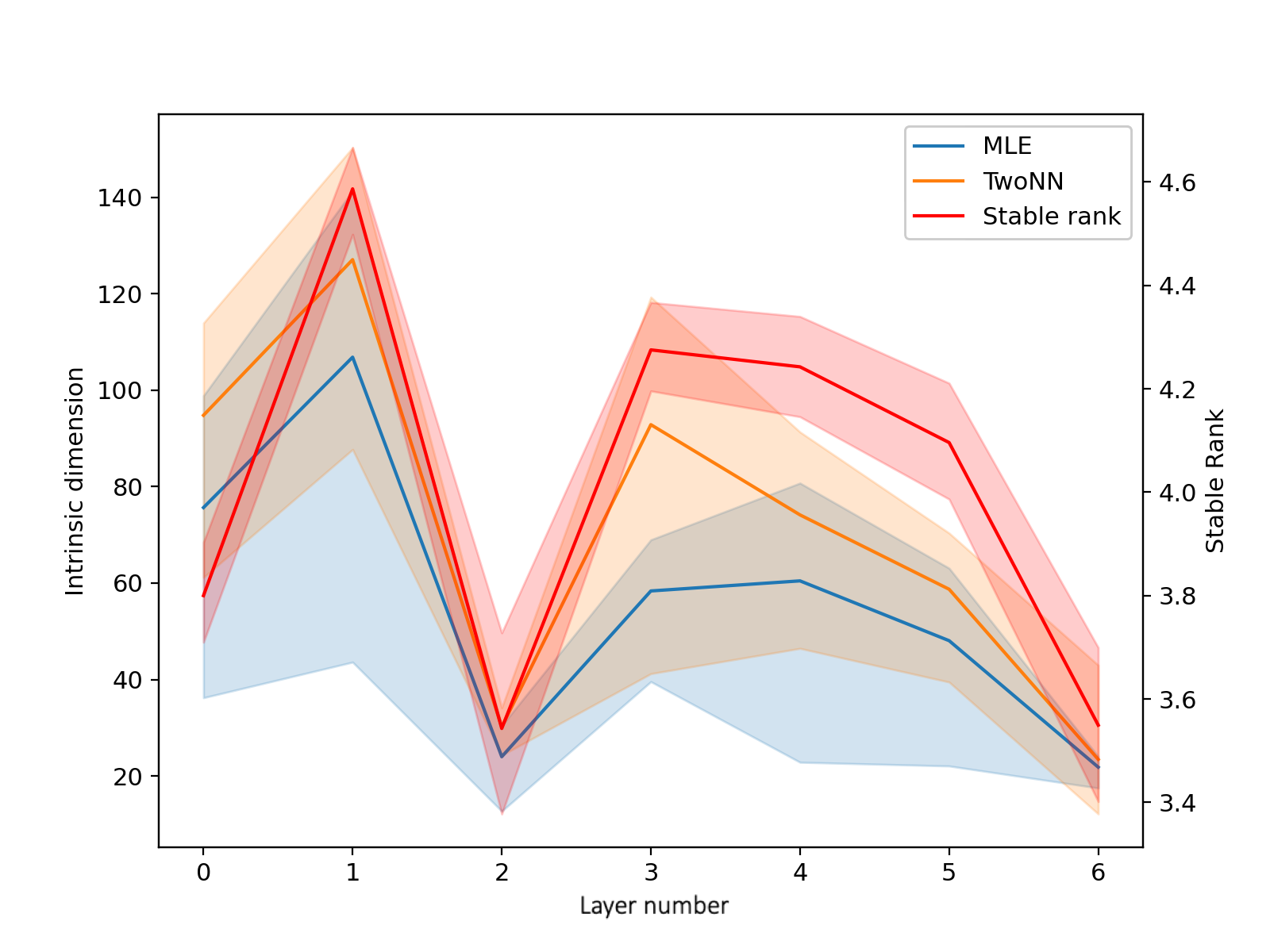}
\end{center}
\caption{An estimation of the intrinsic dimension of the hidden activations of $5,000$ ImageNet images using MLE and TwoNN within a ViT model. We include the stable rank for comparison. Input and output layers are omitted, so $0$ corresponds to the first hidden representation. Shaded regions indicate $95\%$ confidence intervals over $40$ randomly selected ImageNet images for stable rank and three random samplings of $5,000$ ImageNet images for MLE and TwoNN. \label{fig-ID-vs-SR-vit}}
\end{figure}

\subsection{Frame CKA}
\label{sect-frame-CKA}

Let $D = \{x_1,\dots,x_d\}$ be a dataset in $\mathbb{R}^n$ and $F^1,F^2: \mathbb{R}^n \rightarrow \mathbb{R}^k$ two neural networks. Let $F^1_{\leq i_1}(D)$ be the matrix whose rows are $F^1_{\leq i_1}(x_1), \dots, F^1_{\leq i_1}(x_d)$ with analogous notation for $F^2_{\leq i_2}(D)$. Using the notation from Section \ref{sect-neural-frames}, the centered kernel alignment (CKA) \cite{kornblith2019similarity} score for models $F^1,F^2$ at layers $i_1$ and $i_2$ respectively in (terms of $D$) is defined as
\begin{equation*}
\text{CKA}(F^1_{\leq i_1}(D),F^2_{\leq i_2}(D)) = \frac{||\text{Cov}\big(F^1_{\leq i_1}(D),F^2_{\leq i_2}(D)\big)||_F^2}{||\text{Cov}\big(F^1_{\leq i_1}(D),F^1_{\leq i_1}(D)\big)||_F||\text{Cov}\big(F^2_{\leq i_2}(D),F^2_{\leq i_2}(D)\big)||_F}
\end{equation*}
where \(\text{Cov}\) denotes covariance and  $||\cdot||_F$ is the Frobenious norm. Very roughly, CKA measures the structural similarity of representations of datapoints $D$ in $F^1$ at layer $i_1$ vs the representation of $D$ in $F^2$ at layer $i_2$ with higher scores (closer 1) indicating representations are that structrally similar. Notably, CKA is invariant to orthogonal transformation, which fits with the intuition that rotating a model's representation does not meaningfully change its structure (see \cite{klabunde2023similarity} for further discussion on this and other invariances in similarity metrics).

The same reasons that CKA is a useful tool for comparing high-dimensional model representations make it appropriate to comparing neural frames. In particular, if $v_1(x), \dots, v_k(x)$ is a $k$-frame at $x \in \mathbb{R}^n$, then we can apply CKA to the matrices whose rows are
\begin{equation*}
dF^1_{\leq i_1}(v_1(x)), \dots, dF^1_{\leq i_1}(v_k(x)) \quad \text{and}  \quad dF^2_{\leq i_1}(v_1(x)), \dots, dF^2_{\leq i_1}(v_k(x))
\end{equation*}
respectively. We call the resulting statistic the \emph{frame CKA score} of $F^1$ and $F^2$ at layers $i_1$ and $i_2$ for frame $v_1(x),\dots,v_k(x)$. \cwg{Should there be derivatives here? If not, it 
could be simpler to say "apply CKA to the dataset \(D = \{v_1(x), \dots, v_k(x)\}\).} Following the standard interpretation of CKA scores, a frame CKA score close to $1$ indicate that $F^1$ and $F^2$ represent frame $v_1(x), \dots, v_k(x)$ similarly at layers of $i_1$ and $i_2$ respectively. Note that unlike standard CKA which compares the representation of a collection of points, frame CKA compares the arrangement of the vectors of a neural frame at a single point.

\section{What Can Frame CKA Tell Us?}

\begin{figure}[h]
\begin{center}
\hspace{8mm}\includegraphics[width=.35\linewidth]{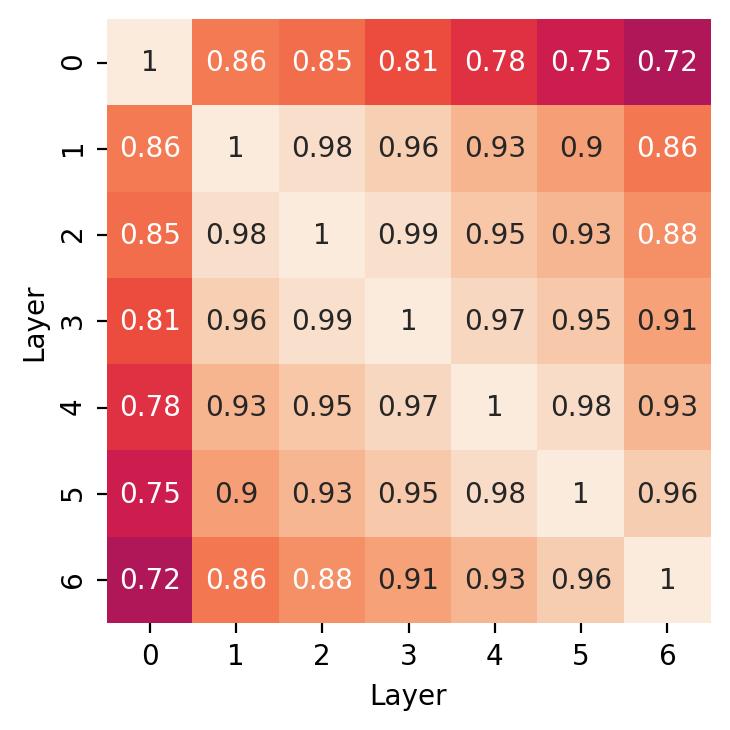}\\
\includegraphics[width=.35\linewidth]{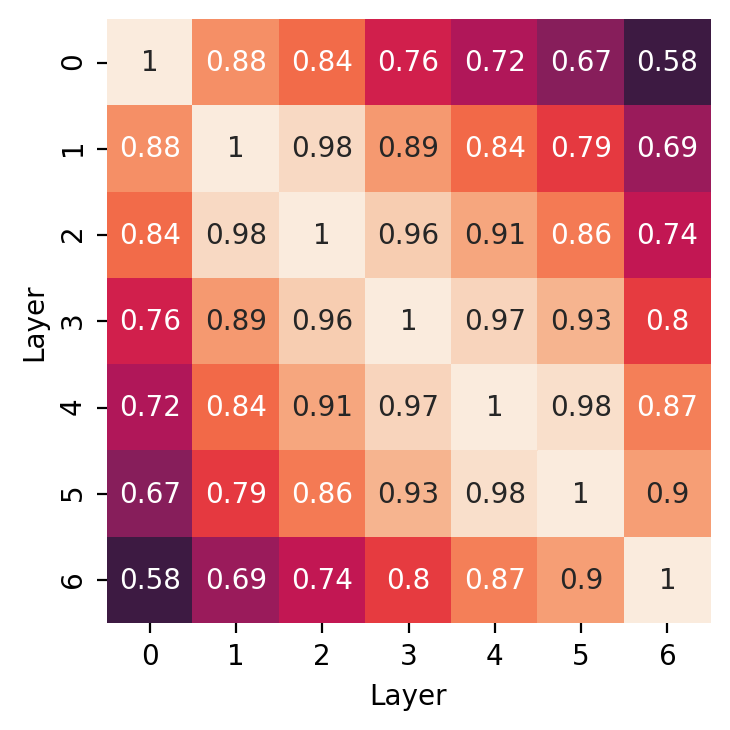}\\

\hspace{5mm}\includegraphics[width=.41\linewidth]{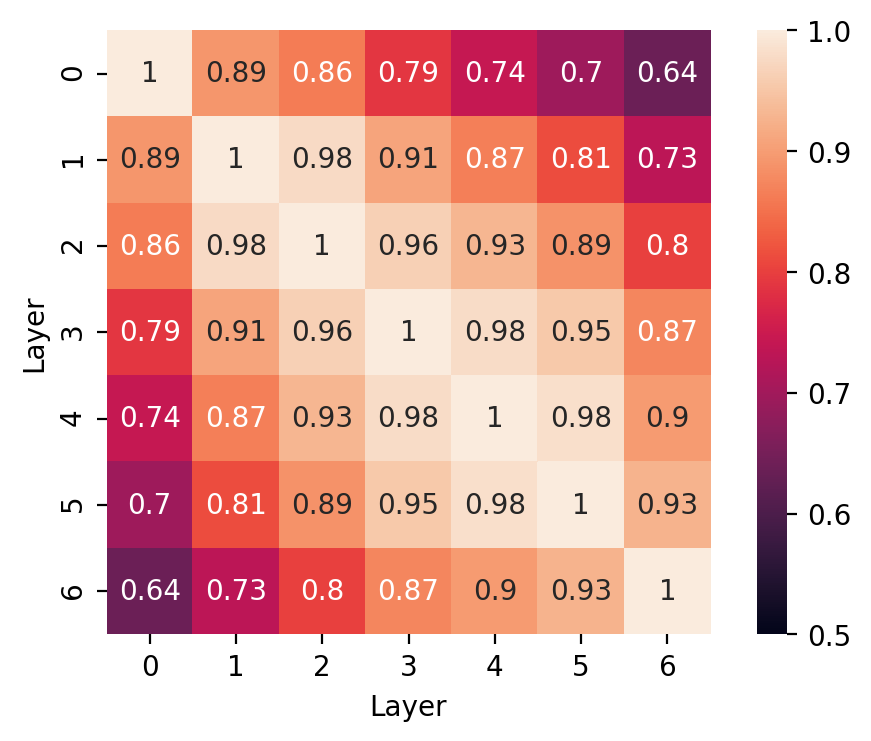}
\end{center}
\caption{The frame CKA scores between layers in the same model when augmentation frames are used. (\textbf{Top}) an ImageNet trained vanilla ResNet50 \cite{marcel2010torchvision}, (\textbf{Middle}) an adversarially trained ImageNet ResNet50 with $\epsilon = 5$, (\textbf{Bottom}) an ImageNet trained ResNet50 with Deep Augmentation. One can see that both heavy augmentation and adversarial training cause a model's representations to increasingly vary between layer. This effect appears to be stronger for adversarial training than heavy augmentation. \label{fig-CKA-plot0}}
\end{figure}

\begin{figure}[h]
\begin{center}
\includegraphics[width=.35\linewidth]{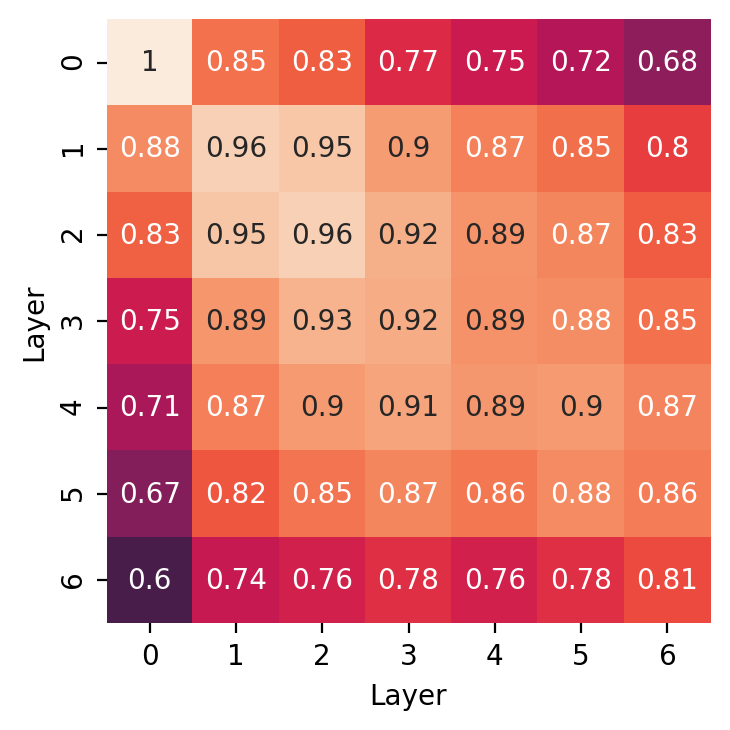}\\
\includegraphics[width=.35\linewidth]{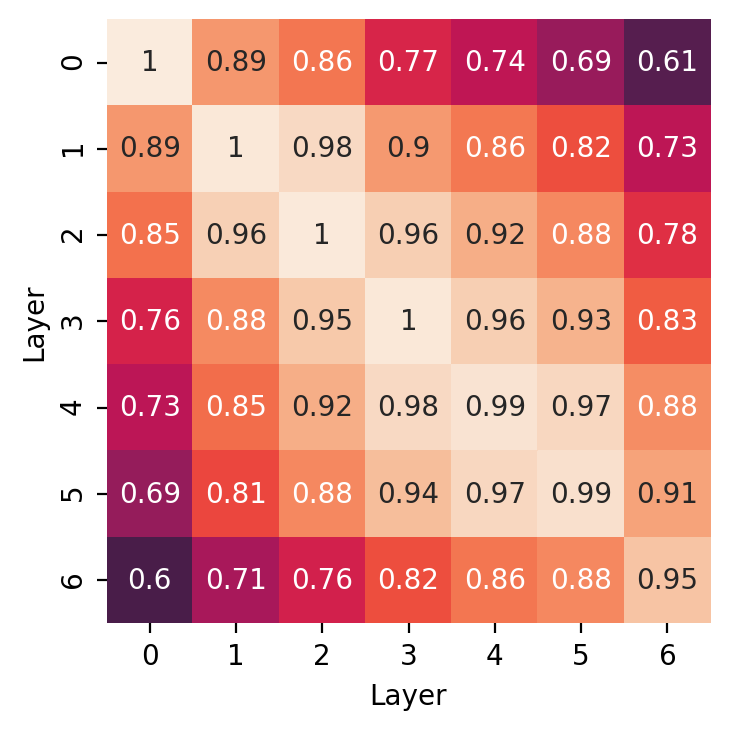}\\
\includegraphics[width=.35\linewidth]{layer_CKA_for_ResNet50_eps5-ResNet50_eps5_annot.png}
\end{center}
\caption{Heatmaps for frame CKA with augmentation frames for three different models (all trained on ImageNet) compared to a ResNet-50 trained on ImageNet with $\epsilon = 5$ adversarial training:  (\textbf{Top}) a ResNet50 trained with $\epsilon = 0.1$ adversarial training, (\textbf{Middle}) a ResNet50 trained with $\epsilon = 3$ adversarial training, (\textbf{Bottom}) and the same ResNet50 trained with $\epsilon = 5$ adversarial training. \label{fig-CKA-plot1}}
\end{figure}

\begin{figure}[h]
\begin{center}
\includegraphics[width=.35\linewidth]{layer_CKA_for_ResNet50v1-ResNet50v1_annot.png}\\
\includegraphics[width=.35\linewidth]{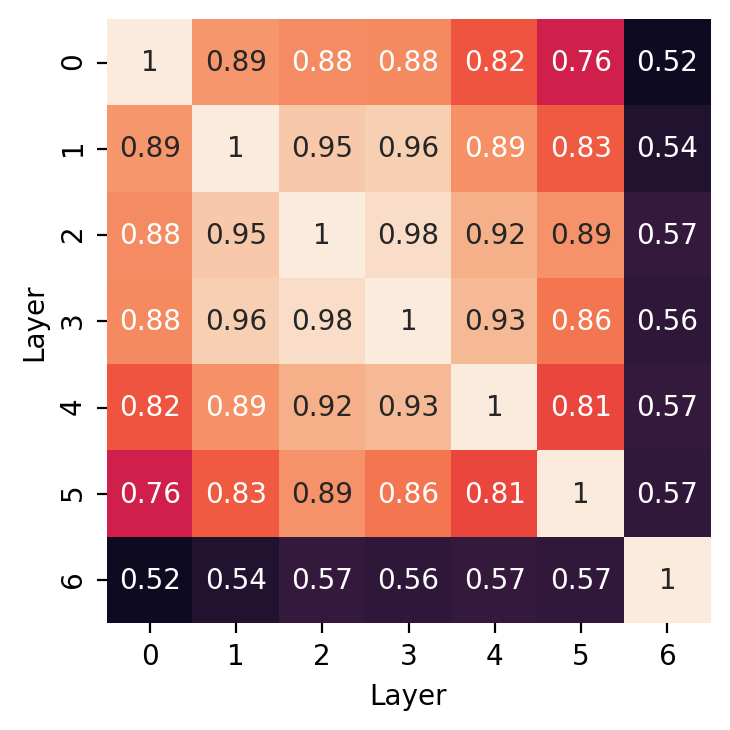}
\end{center}
\caption{Heatmaps for frame CKA with augmentation frames for ResNet50 vs. ResNet50 trained on ImageNet (\textbf{Top}) and ViT vs ViT \cite{dosovitskiy2020image} trained on ImageNet (\textbf{Bottom}). \label{fig-CKA-plot2}}
\end{figure}

In Section \ref{sect-experiments} we noted that frame CKA can illuminate some small-scale properties of a model's representations that are hard to detect with standard CKA. In Figure \ref{fig-CKA-plot1} we show that frame CKA detects the difference in representations between models with adversarial training and models without. Namely, neural frame CKA exhibits that inter-architecture similarity increases with adversarial robustness, in the sense that similarity between layers of different models increases as the $\epsilon$ used in adversarial training increases. For example, the similarity between layers is larger between $\epsilon=3$ and $\epsilon=5$ than between $\epsilon=0.1$ and $\epsilon=5$.\cwg{fix when plots are updated} This was only previously shown with an expensive deconfounding variant of CKA in \cite{jones2022if}. 

In another example, we compare the intra-layer similarity between a ResNet50 and a Vision Transformer. Prior work with CKA had noted that CKA tends to show significant differences between blocks in ResNets but less intra-layer differences in vision transformers. In Figure \ref{fig-CKA-plot1} we show that frame CKA (using augmentation frames) picks up a different signal than standard CKA. We see that frame CKA actually shows greater differences between layers in the ViT suggesting that the outcome of a comparison of these two model types may depend on the scale at which one compares their representations. We note that our observations are consistent with findings in Section \ref{appendix-architecture}.

\section{How does the choice of $k$ impact the stable rank of a frame?}

In our experiments above, we mostly restricted ourselves to $19$-frames as this was the total number of augmentations that we found that were suitable for use in an augmentation frame. It is worth asking what happens when we vary $k$ in a $k$-neural frame. Do our conclusions remain stable? In Figure \ref{fig-vary-k} we show the result of decreasing $k$ for the augmentation frames (described in Section \ref{appendix-augmentation-frames}) for a ResNet50 trained on ImageNet. We find that increasing the value of $k$ in augmentation $k$-frames increases the stable rank of the corresponding neural frames. Nevertheless, the shape of the curves (e.g., layers where stable rank increases) seems to mostly remain the same after sufficiently large $k$. On the other hand, increasing $k$ when using noise frames does not appreciably change the stable rank of the corresponding neural frames (though the stable rank of the input increases predictably with the number of frames). 

The fact that increasing $k$ increases the stable rank of neural augmentation frames but does not increase the stable rank of neural noise frames reinforces the idea that for most models, directions of change associated with augmentation are individually preserved (hence adding them to the input frame causes changes to the corresponding neural frame), whereas noise directions mostly are not. In future work it would be interesting to understand how the addition of specific augmentation directions impact stable rank. Overall, it appears that qualitative patterns in stable rank per layer are mostly preserved when $k$ is changed provided that $k$ is sufficiently large.

\begin{figure}[h]
\begin{center}
\includegraphics[width=.85\linewidth]{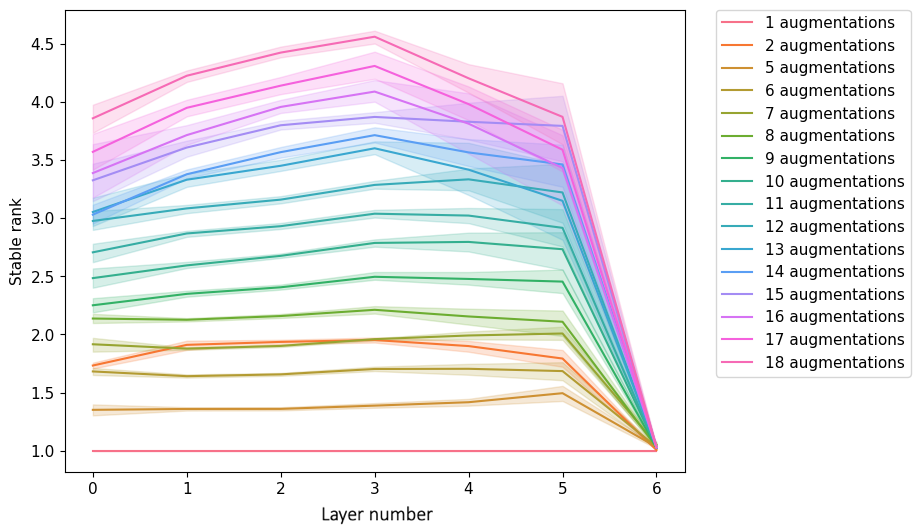}
\end{center}
\caption{The stable rank (by layer) of an ImageNet trained ResNet50 \cite{marcel2010torchvision} evaluated with respect to augmentation frames with varying $k$. The order in which specific augmentations were added was chosen randomly. \label{fig-vary-k}}
\end{figure}

\begin{figure}[h]
\begin{center}
\includegraphics[width=.85\linewidth]{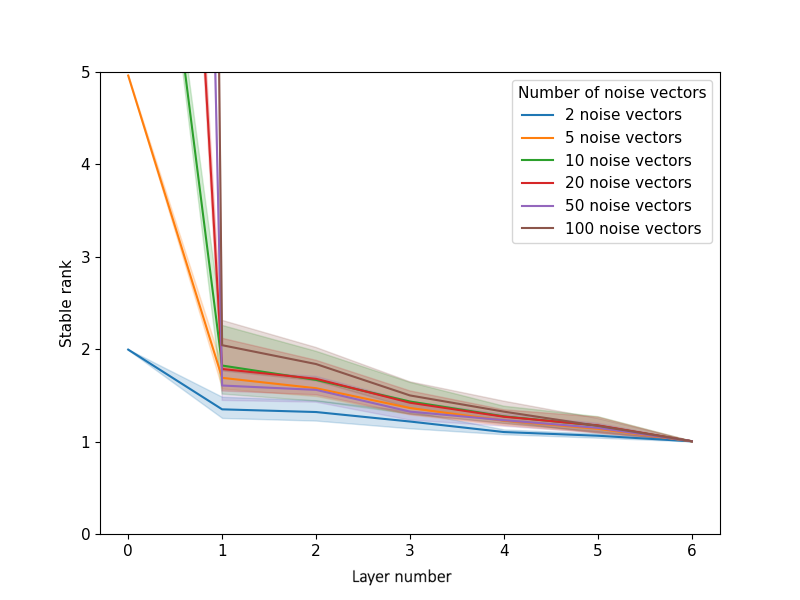}
\end{center}
\caption{The stable rank (by layer) of an ImageNet trained ResNet50 \cite{marcel2010torchvision} evaluated with respect to noise frames for varying $k$. \label{fig-vary-k-nose}}
\end{figure}

\section{Are different types of frames processed differently in a transformer architecture?}

In Section \ref{sect-experiments} we showed that different types of frames are processed very differently by a ResNet50. One might ask if a similar statement holds for vision transformers, which have substantially different architecture. In Figure \ref{fig-stable-rank-different-frames-vit} we show a similar plot to that found in Figure \ref{fig-stable-rank-different-frames-resnet50}. We see a similar phenomenon holds with minor differences. For example, the vision transformer tends to preserve a noise frame somewhat longer than the ResNet50 does. Given the results in Section \ref{sect-experiments}, we speculate that this may relate to vision transformer’s purported adversarial robustness \cite{paul2022vision}.

\begin{figure}[h]
\begin{center}
\includegraphics[width=.85\linewidth]{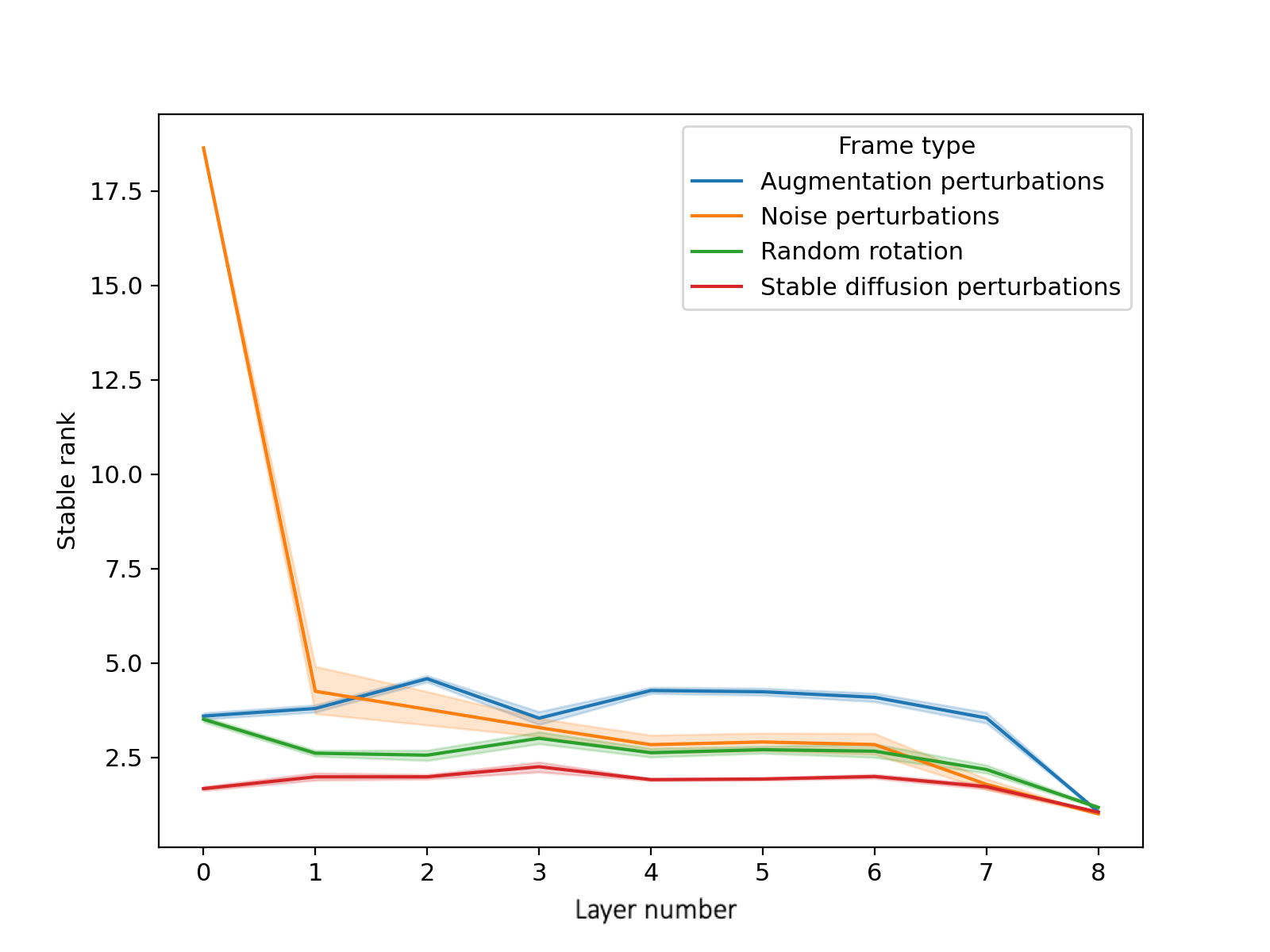}
\end{center}
\caption{The stable rank of different types of frames measured at various layers of a ViT model. The ResNet50 version of this plot is found in Figure \ref{fig-stable-rank-different-frames-resnet50}. Layer zero corresponds to model input and the last layer corresponds to model output. Shaded regions indicate $95\%$ confidence intervals over $40$ randomly selected ImageNet images. \label{fig-stable-rank-different-frames-vit}}
\end{figure}


\section{Stable rank over the course of training}

To better understand how the stable rank of a frame changes over the course of training, we saved the weights of a ResNet18 \cite{he2016deep} trained from scratch on ImageNet every $10$ iterations (for $1000$ iterations) and then every $100$ iterations for the approximately $16$ remaining epochs. The training hyperparameters that we used can be found in Table \ref{table-hyperparameters}. 

In Figure \ref{fig-stable-training} we show the stable rank (by layer) for this ResNet18 with respect to an augmentation frame at different stages of training. We see that at a large scale the general trend is for stable rank to increase as training increases, but that these changes are most significant in the later layers of the model. For example, the latent space layer (layer 5), increases from an initial stable rank around $1.5$ to a stable rank of $3.5$, an increase of $2$, while the stable rank of layer 1 (in one of the first blocks of the model), only increases from $3.5$ to $4$, an increase of only $.5$. We conjecture that one effect of the later stages of training is that a model gains the tendency to preserve those frames related to natural changes of an image. This guess is supported by Figure \ref{fig-training-early} (right) which shows that while stable rank increases throughout training for frames of naturalistic directions, it decreases for noise frames whose directions lack any connection to the content of the image.

Interestingly, we find stable rank also peaks (though not as high as later) once in the early iterations of training. In Figure \ref{fig-training-early} (left) we see that stable rank increases for approximately the first 50 iterations of training and then decreases again until around iteration 200. It then slowly increases for the rest of training. It would be interesting to understand what drives these dynamics.

\begin{figure*}[h]
\begin{center}
\includegraphics[width=.75\linewidth]{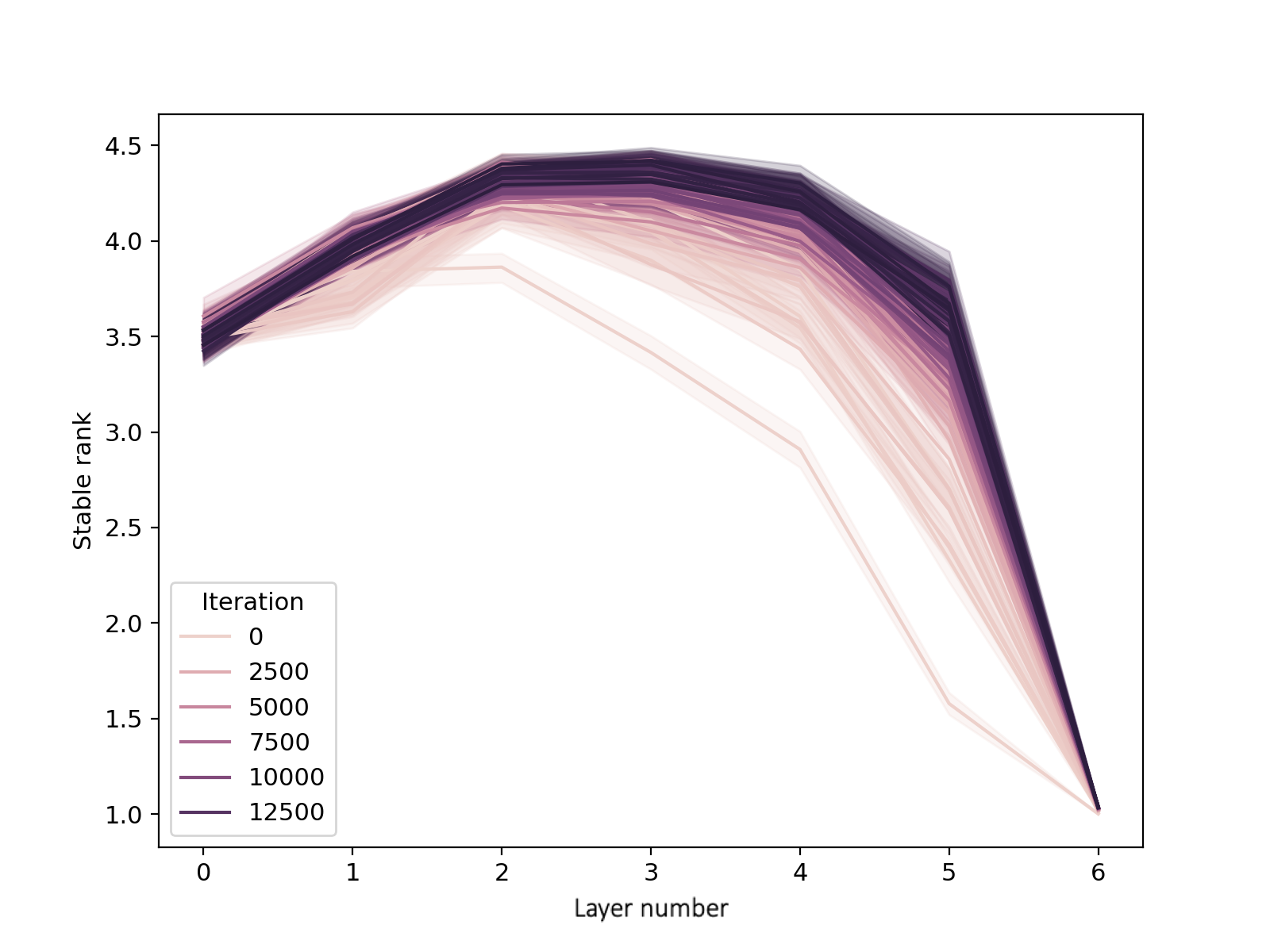}
\includegraphics[width=.75\linewidth]{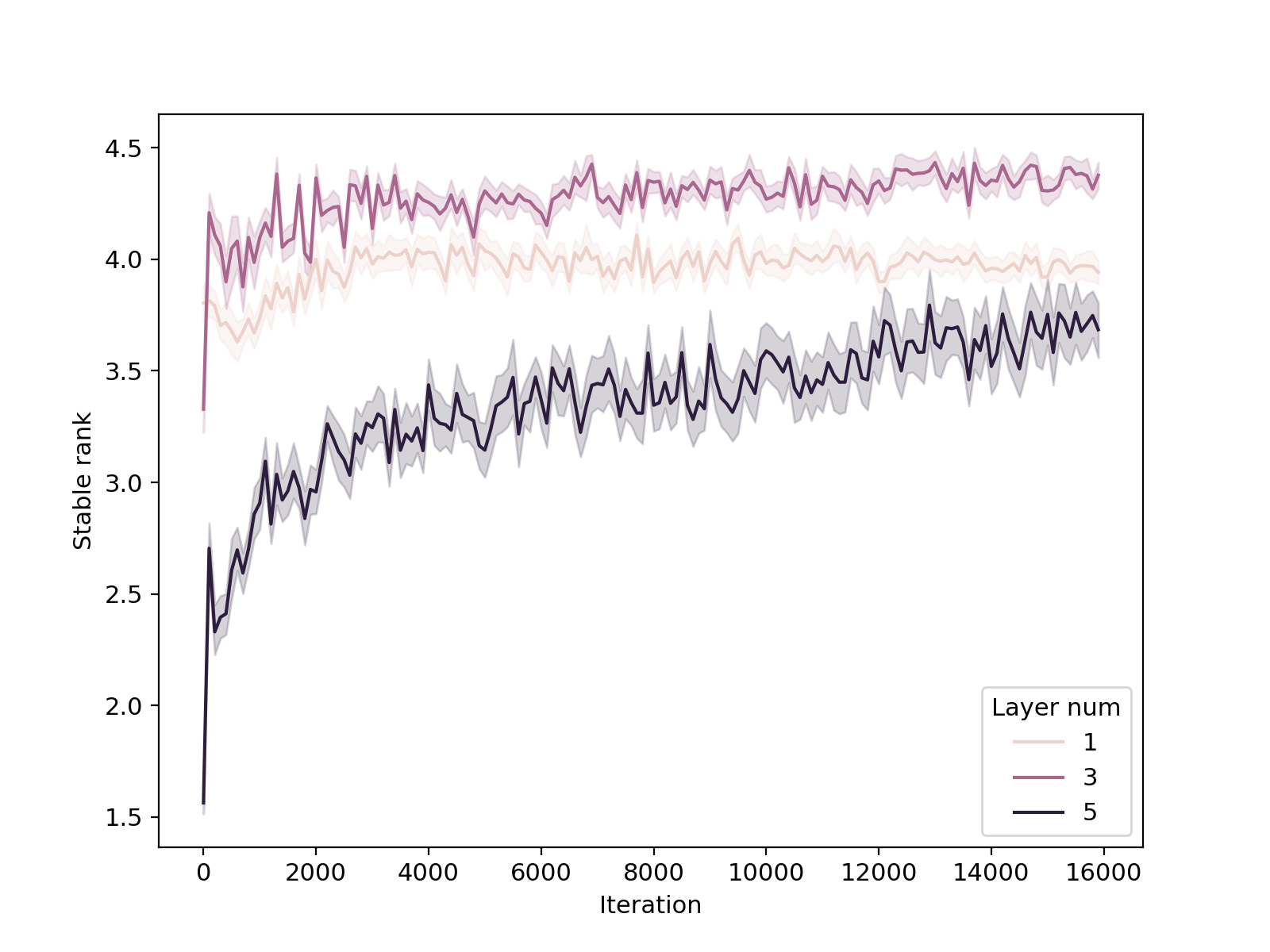}
\end{center}
\caption{\textbf{(Left)} The stable rank of an augmentation frame (by layer) for a ResNet18 trained from scratch. Different colored curves correspond to the number of iterations of training that the model has undergone. {\textbf{(Right)}} The stable rank of different layers of the model as a function of the number of training iterations. Shaded regions in both plots indicate $95\%$ confidence intervals over $40$ random ImageNet images. \label{fig-stable-training}}
\end{figure*}

\begin{figure*}[h]
\begin{center}
\includegraphics[width=.75\linewidth]{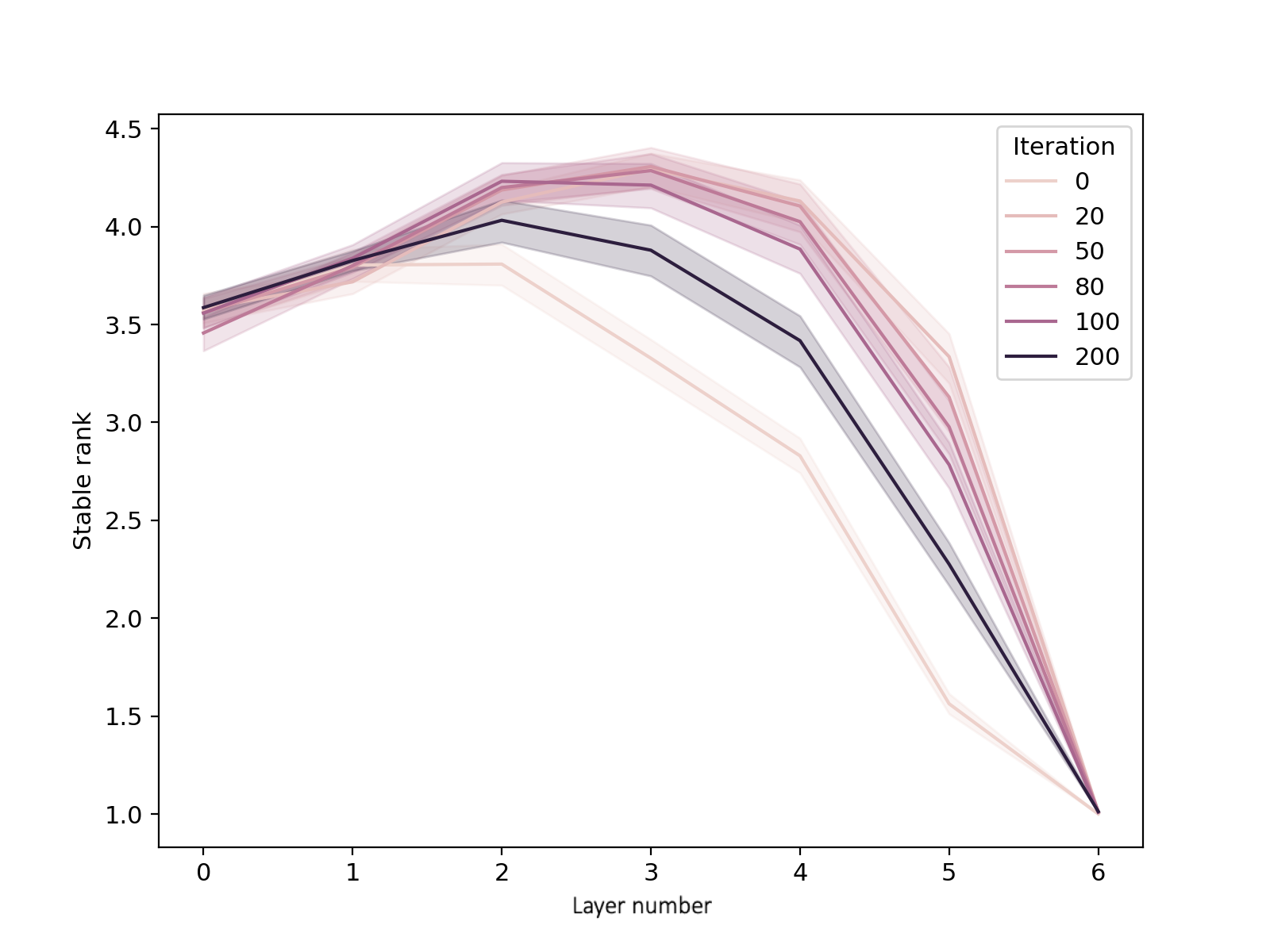}
\includegraphics[width=.75\linewidth]{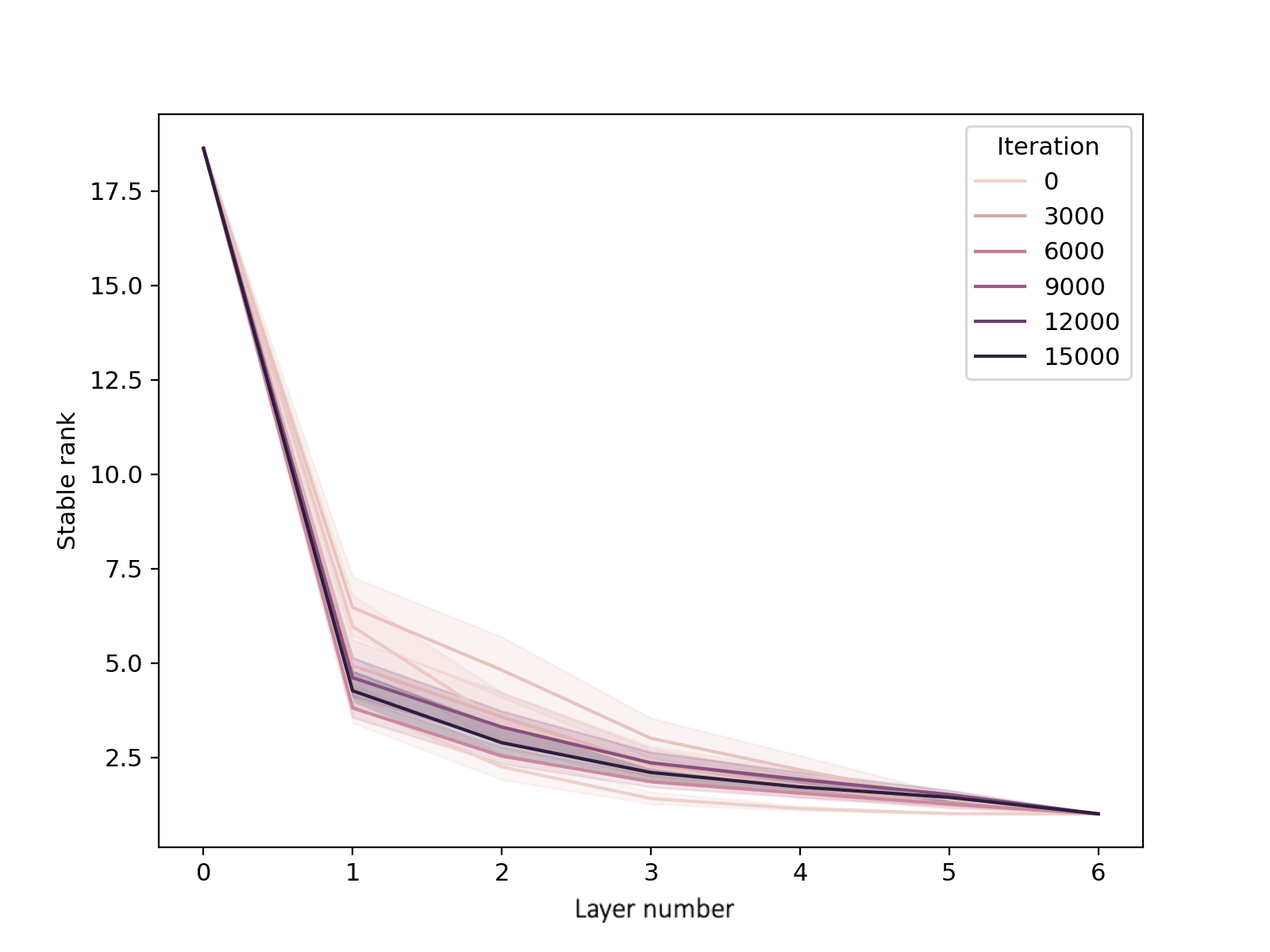}
\end{center}
\caption{\textbf{(Left)} The stable rank of an augmentation frame (by layer) during early iterations of a  ResNet18 trained from scratch. {\textbf{(Right)}} The stable rank of the ResNet18 model (by layer) evaluated on a noise frame. In both plots, different colored curves correspond to the number of iterations of training that the model has undergone.\label{fig-training-early}}
\end{figure*}

\section{Stable rank and architecture}
\label{appendix-architecture}

In Figures \ref{fig-stable-rank-architectures-CNN} and \ref{fig-stable-rank-architectures-ViT} we plot the stable rank (as a function of layer) for an augmentation frame and two different families of architectures: CNNs and vision transformers. Shaded regions depict $95\%$ confidence intervals calculated over $40$ random ImageNet images. The CNN architectures that we plot (left) are DenseNet121 \cite{huang2017densely}, InceptionV3 \cite{szegedy2016rethinking}, ResNet50 \cite{he2016deep}, and ResNeXT50 \cite{xie2017aggregated}. On the right we plot hidden layers from transformers: ViT \cite{dosovitskiy2020image} and Swin \cite{liu2021swin}. All use the default ImageNet torchvision \cite{marcel2010torchvision} weights.
We note two trends in these plots: 
\begin{enumerate}[(i)]
    \item \label{item:plateau} All curves consist of a plateau spanning most layers of the model followed by a dramatic dropoff in stable rank at the last layers.
    \item \label{item:osc} The transfomer models exhibit significantly more \emph{fluctuation} in stable rank than the CNNs. 
\end{enumerate}
\Cref{item:plateau} could be partially explained by the fact that \emph{all} models studied in this paper, both CNNs and transformers, include residual connections. Note that a toy residual network with \(n\)-dimensional feature spaces and identity activations consists of a composition of layers of the form \( I_n + W \), where \(I_n\) is an identity matrix and \(W\) a \(n\times n\) weight matrix. These have singular values of the form \(1 + \sigma_i\), where \(\{\sigma_i\}\) are the singular values of \(W\), and thus stable rank 
\begin{equation}
\label{eq:res-stable-rank}
    \frac{1}{(1+ \sigma_1)^2} \sum_i (1+ \sigma_i)^2.
\end{equation}
Suppose \(W\) is a random matrix with IID entries sampled from \(\mathcal{N}(0, \frac{2}{n})\) (this is true before training with He normal initialization \cite{he2016deep}). Then a calculation using \cite{mpdist} shows that for large \(n\)  the expected stable rank of \(I_n + W \) is approximately
\begin{equation}
    \frac{n}{(1+ 2 \sqrt{2})^2} \int_0^{2\sqrt{2}} (1+y^2)\sqrt{8-y^2} \frac{dy}{2 \pi} \approx 0.37 n.
\end{equation}
Provided this is larger than the stable rank of the neural frame in the preceding layer, Lemma \ref{lem:stable-rank-noninc} might suggest that \(I_n + W \) preserves the stable rank of the neural frame.\footnote{The expected stable rank of \(W \) itself is \(\frac{n}{4} \).}

As for \cref{item:osc}, while these experiments alone are insufficient to identify a reason for this apparent difference, one could make a number of different conjectures. It could be, for instance, that the priors hardcoded into CNNs dampen the extent to which the models stretch and compress an image manifold. Alternatively, it might be that the nonlinearity of attention layers leads to more geometric changes layer-to-layer. This nonlinearity is in contrast to the convolutions in CNNs, which are linear in isolation (that is, not considering the nonlinear layers that often follow them). These questions would be interesting to address in follow-up work.\\

\begin{figure*}[h]
\begin{center}
\includegraphics[width=.75\linewidth]{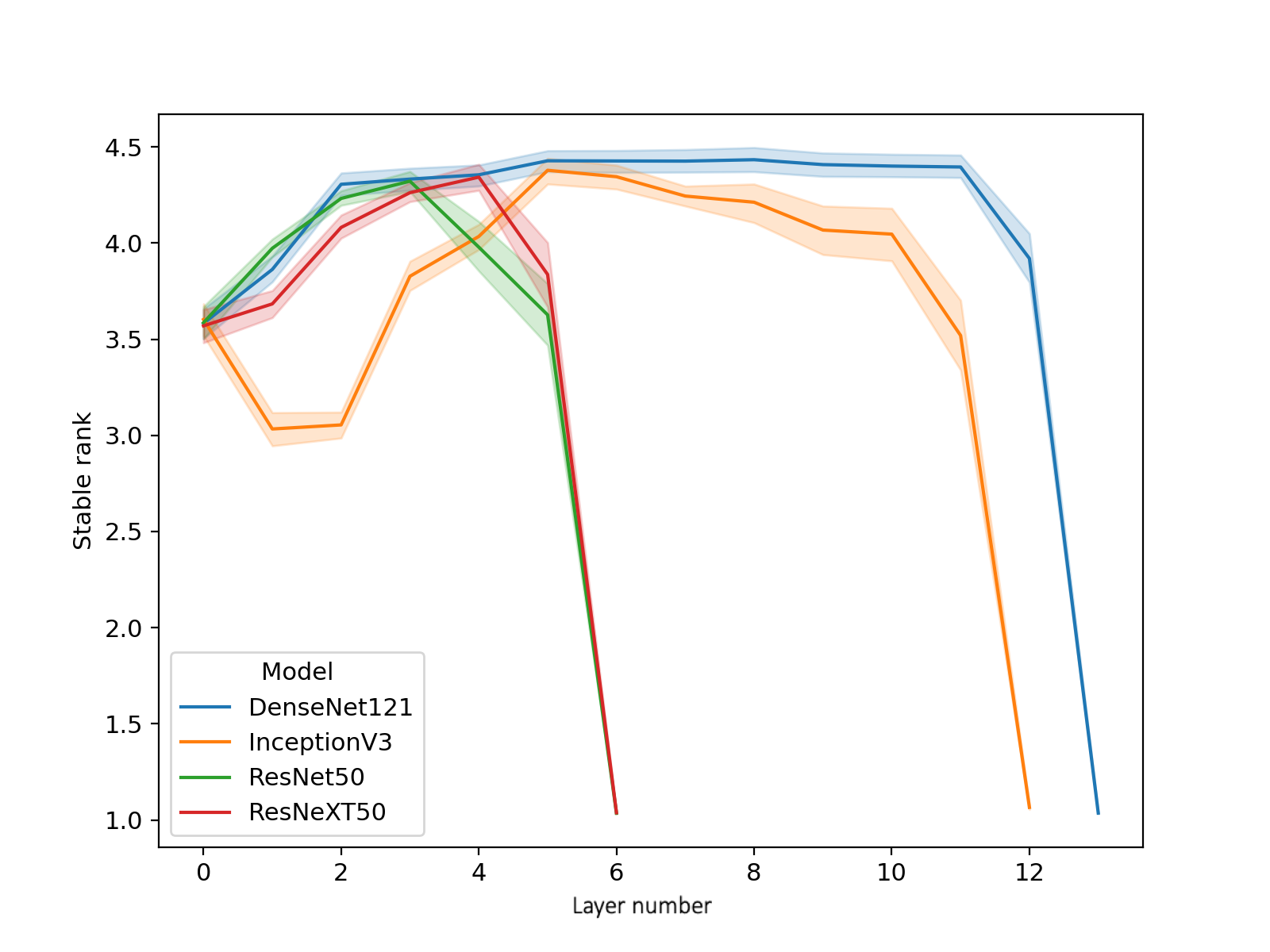}
\end{center}
\caption{The stable rank, by layer, for a range of different CNN architectures. Shaded regions indicate $95\%$ confidence intervals over $40$ random ImageNet images. \label{fig-stable-rank-architectures-CNN}}
\end{figure*}

\begin{figure*}[h]
\begin{center}
\includegraphics[width=.75\linewidth]{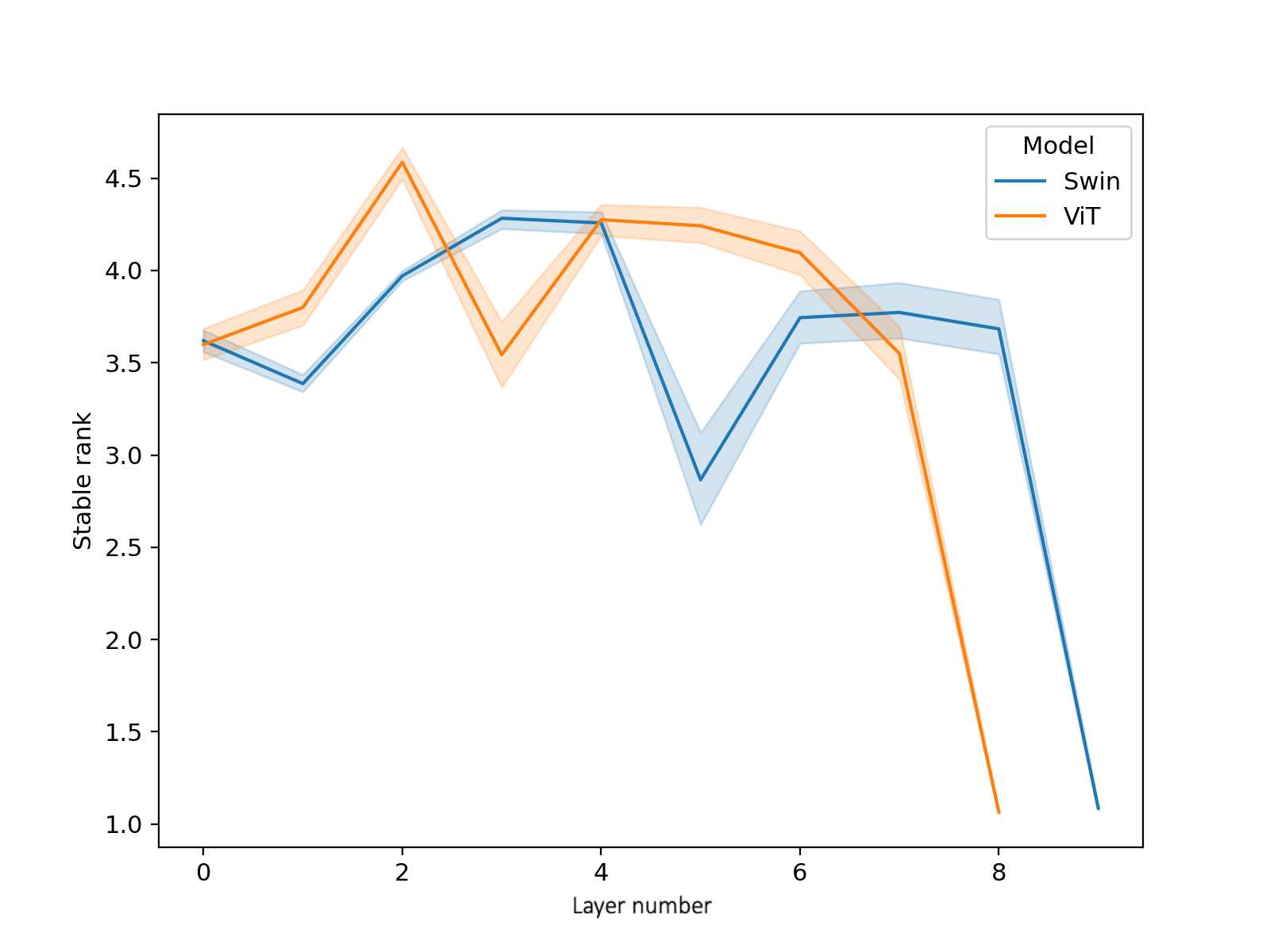}
\end{center}
\caption{The stable rank, by layer, for a range of different transformer architectures. Shaded regions indicate $95\%$ confidence intervals over $40$ random ImageNet images. \label{fig-stable-rank-architectures-ViT}}
\end{figure*}

\section{Geometric Background}
\label{appendix-geo-background}

In this paper we take the word `manifold' to mean a smooth manifold in the formal sense. A comprehensive reference on manifolds is \cite{lee2013smooth} --- here for convenience we briefly introduce key geometric objects of interest: tangent bundles, their sub-bundles, and frames.

Let $M \subset \mathbb{R}^n$ be an $m$-dimensional smooth manifold, $x$ a point on $M$, and suppose $\gamma: (-1, 1) \rightarrow M$ is a smooth path on $M$ such that $\gamma(0) = x$. The {\emph{tangent vector}} associated with $\gamma$ at $x$ is the derivative of $\gamma$ at $0$, $\gamma'(0)$. The tangent space $T_xM$ of $M$ at $x$ is the vector space of all such tangent vectors \(\gamma'(0)\) at $x$ (here \(\gamma\) varies over all possible smooth paths in $M$ that pass through $x$). The {\emph{tangent bundle}} $TM$ of $M$ is the union of all tangent spaces for each $x \in M$, $\coprod_{x \in M} T_xM$ --- it is a manifold in its own right of dimension $2m$. For more details see Section 3 in \cite{lee2013smooth}.

The tangent bundle can be thought of as the assignment of a vector space to each point in $M$; the concept of a vector bundle generalizes the tangent bundle of a manifold. Informally, a vector bundle is a smooth map $\pi: E \rightarrow M$ such that the fibers \(E_x := \pi^{-1}(x) \subseteq E\) are finite-dimensional real vector spaces isomorphic to $\mathbb{R}^l$ and they vary smoothly with respect to \(x\) in the sense that for any $x \in M$ there is a neighborhood $U$ of $x$ with a diffeomorphism $\varphi_x: \pi^{-1}(U) \to U \times \mathbb{R}^l$. A sub-vector bundle \(F \subseteq E \) is an embedded submanifold which is itself a vector bundle over \(M\) (with respect to the induced smooth map \(F \subseteq E \xrightarrow{\pi} M\)) such that for each point \(x\in M\), \(F_x \subseteq E_x\) is a linear subspace. For formal definitions we refer to Section 10 in \cite{lee2013smooth}. Our interest is in sub-vector bundles of tangent bundles $TM$.



A common way to obtain such sub-vector bundles in practice is from vector fields. Recall that in the case of the tangent bundle, the map $\pi: TM \rightarrow M$ is the map such that if $z \in T_xM \subset TM$, then $\pi(z) = x$. Then a smooth vector field is formally a smooth function \(v: M \to TM \) with the property that \(\pi(v(x)) = x \) for all \(x \in M\) (informally this corresponds to the usual notion that a vector field consists of a choice of tangent vector for each $x \in M$). We can use parametrized functions from our manifold to itself to construct vector fields.

\begin{lemma}[{Prop 9.7 \cite{lee2013smooth}}]
\label{lem:flows}
If $f: (-1,1) \times M \rightarrow M$ is a smooth function on smooth manifold $M$ with the property that \(f(0, x) = x\) for all \(x \in M\), then the function \(v(x) = \frac{\partial}{\partial t}f(t,x)|_{t=0}\) is a smooth vector field.
\end{lemma}

\begin{lemma}
\label{lem:non-vanishing-determinants}
Let \(v_1, \dots, v_k : M \to TM \) be smooth vector fields on an $m$-dimensional manifold $M$. 
\begin{enumerate}[(i), nosep]
    \item The set $U$ of all $x$ in $M$ such that $v_1(x), \dots, v_k(x)$ are linearly independent is open.\footnote{although possibly empty.}
    \item The (sub)spaces \(\mathrm{span}(v_1(x), \dots, v_k(x)) \subseteq T_x M\) form a sub-vector bundle of the tangent bundle of \(U\).
\end{enumerate}
\end{lemma}

We put the two lemmas above together to give the statement that will form the basis for one type of neural frames that we introduce in the next section.

\begin{corollary}
\label{cor-bundle-from-functions}
Suppose that $\mathcal{F} = \{f_i: (-1,1)\times M \to M \, | \, i = 1,\dots, k\}$ is a collection of smooth maps such that $f_i(0,x)=x$ for all \(x \in M\), and let $v_i(x) = \frac{\partial}{\partial t}f_i(t,x)|_{t=0}$. If $v_1(x), \dots, v_k(x)$ are linearly independent in $T_xM$ for all $x \in M$, then $f_1, 
\dots, f_k$ define a $k$-dimensional vector bundle on $M$ which we denote by $V_{\mathcal{F}}$ and $v_1(x), \dots, v_k(x)$ is a $k$-frame of this vector bundle.
\end{corollary}

The geometric machinery we have introduced in this section will provide the framework for neural frames, which we introduce below. We note however that this framework is supposed to act as a guide, not a guarantee. Indeed, by necessity, we will have to violate certain assumptions when running experiments. For example, many popular deep learning architectures are not actually smooth everywhere and some of our augmentations will not be smooth either (largely due to the discrete nature of digital images). However, we have tried to choose functions that are at least moderately well-behaved.

\section{Proofs}

The statement of \cref{lem:flows} and the proof below is very similar to those of Proposition 9.7 in \cite{lee2013smooth}, however note that our hypotheses on \(f\) are weaker.
\begin{proof}[Proof of \cref{lem:flows}]
First, by Proposition 3.14 in \cite{lee2013smooth} there is a natural decomposition of tangent bundles
\begin{equation}
    T \big((-1, 1) \times M\big) \simeq T (-1, 1) \times TM
\end{equation}
and so by Proposition 3.21 in \cite{lee2013smooth} the differential of \(f\) can be naturally identified as a smooth map 
\begin{equation}
    df: T (-1, 1) \times TM \to TM.
\end{equation}
Now as \(T (-1, 1) \simeq (-1, 1)\times \RR\) (for example combining Proposition 3.9 and 3.13 in \cite{lee2013smooth}), letting \(z : M \to TM\) denote the zero section we can define a smooth map \(\frac{\partial f(t, x)}{\partial t}|_{t=0}\) as
\begin{equation}
    \begin{split}
        &M \xrightarrow[]{\sigma}  (-1, 1)\times \RR \times TM  \simeq  T (-1, 1) \times TM  \xrightarrow[]{df} TM  \\
        &\text{sending } x \mapsto (0, 1, z(x)) \mapsto df((0, 1), z(x))
    \end{split}
\end{equation}
(one can check in coordinates that this is indeed a partial derivative with respect to \(t\), hence the notation). It remains to check that if \(\pi: TM \to M\) is the projection, we have \( \pi(\frac{\partial f(t, x)}{\partial t}|_{t=0}) = x \), and this follows from the hypothesis that \(f(0,x) = x\) for all \(x \in M\).
\end{proof}

\begin{proof}[Proof of \cref{lem:non-vanishing-determinants}]
The statement is local on \(M\). We may therefor assume that \(M \subseteq \RR^m \) is an open subset of some euclidean space, in which case we have a canonical identification \(TM \simeq M \times \RR^m\). Since a section of the projection \(\pi: M \times \RR^m \to  M\) is equivalent to a function \(M \to \RR^m \), we may now identify \(v_1, \dots, v_k \) as smooth functions
\begin{equation}
    \begin{split}
        &v_1, \dots, v_k: M \to \RR^m, \text{ whose product is a smooth map} \\
        & (v_1, \dots, v_k): M \to (\RR^m)^k
    \end{split}
\end{equation}
Finally, let \(m_J: (\RR^m)^k \to \RR\) be the \(k \times k \) minors of \(m \times k\) matrices, where \(J \subset \{1, \dots, m\} \) ranges over subsets of cardinality \(k\) --- these are polynomials and hence smooth. The set 
\begin{equation}
    \begin{split}
        &\{(w_1, \dots, w_k) \in (\RR^m)^k \\
        & \, | \, w_1, \dots, w_k \text{ are linearly independent} \}
    \end{split}
\end{equation}
can be identified as the locus where \( \prod_J m_J (w_1, \dots, w_k) \neq 0\); since the non-vanishing locus of a smooth function is open, we have proved part (i) of the lemma.

For part (ii), by the very definition of \(U \) the functions \(v_1, \dots, v_k \) define a global isomorphism 
\begin{equation}
    \begin{split}
        &\varphi: U\times \RR^k \to \mathrm{span}(v_1, \dots, v_k) \\
        &\text{where } \varphi(x, (c_1, \dots, c_k)) = \sum_i c_i v_i (x).
    \end{split}
\end{equation}
Using this, it is straightforward to check that \(\mathrm{span}(v_1, \dots, v_k) \) is a vector bundle over \(U \).
\end{proof}

\begin{remark}
All of the above proofs work even if the \(v_i \) are continuous (smoothness is not required). The proof of (ii) in fact shows the stronger statement that \(\mathrm{span}(v_1, \dots, v_k)\) is a \emph{trivial} vector bundle over \(U\). It is worth mentioning that in the \emph{global} context where \cref{lem:non-vanishing-determinants} is stated, there are examples where it is impossible that \(U = M \), even when \(k = 1\): for example, the famous ``hairy ball theorem'' says that when \(M = S^2 \) (a \(2\)-sphere), every globally defined vector field \(v: M \to TM \) vanishes at at least one point.  
\end{remark}

We sketch here a statement making precise the sense that for almost all sets of \(k \leq \dim M\) vector fields \(v_1, \dots, v_k: M \to TM \) the locus 
\begin{equation}
\label{eq:non-degen-locus}
    \{ x\in M \, | \, v_1(x), \dots, v_k(x) \text{ are linearly independent} \} 
\end{equation}
is a dense open set with measure 0 complement. Let  \(\Gamma(M, TM) \) be the real vector space of vector fields on \(M\). We will make use of the \emph{relative} \(k\)-fold product \( \prod_{M, i=1}^k TM\) of \(TM \) over \(M\); this is simply the space of \(k\)-tuples of tangent vectors \(w_1, \dots, w_k \in TM\) such that \(\pi(w_i)=\pi(w_j) \in M\) for all \(i, j\). The relative product contains a subspace \(N \subseteq \prod_{M, i=1}^k TM\) consisting of linearly dependent \(k\)-tuples \((w_1, \dots, w_k)\). Note that this subspace is \emph{not} a submanifold: on each fiber \(\prod_{i=1}^k T_xM\) it is a \emph{vanishing locus} of a product \(\prod_J m_J (w_1, \dots, w_k)\) of minors as appearing in the above proof.\footnote{That is, one can verify that \(N\) is singular using the Jacobian criterion.} However, it does admit a \emph{stratification} 
\begin{equation}
    N_0 \subseteq N_1 \subseteq \dots, \subseteq N_{k-1} = N
\end{equation}
such that \(N_{i} \setminus N_{i-1} \) is a (not necessarily closed) submanifold of \( \prod_{M, i=1}^k TM\) for \(i=1,\dots, k-1\). Namely, one defines \(N_i \) to consist of the \((w_1, \dots, w_k)\) whose associated \(m\times k \) matrix has rank less than or equal to \(i\). 

\label{claim:ae}
\textbf{Claim:} Let \(V \subset \Gamma(M, TM) \) be a finite dimensional linear subspace,  and assume that for each \(x \in M \) the natural linear map \(\Psi: V \to T_x M \) defined as \(\Psi(v) = v(x)\) is surjective. Then, for almost every \((v_1,\dots, v_k ) \in V^k \) the set defined in \cref{eq:non-degen-locus} is a dense open set with measure 0 complement.

Note that in Claim \ref{claim:ae}, the fact that $\Psi$ is surjective for each $x \in M$ does not imply that $V$ is equal to $\Gamma(M,TM)$ (for example, $\Gamma(M,TM)$ will generally be infinite dimensional). Rather, for any $v \in T_xM$ there is a vector field that takes value $v$ at $x$.
We will not provide a full proof of this claim, merely a sketch. To begin, consider the natural map 
\begin{equation}
    \Phi: V^k \times M \to \prod_{M, i=1}^k TM
\end{equation}
defined by \(\Phi((v_1, \dots, v_k), x) \mapsto (v_1(x), \dots, v_k(x))\). \emph{If} \(N \) \emph{were} a smooth closed submanifold of \( \prod_{M, i=1}^k TM\) (as stated above, it isn't, this is just a thought experiment to motivate a proof sketch), then we could proceed as follows: the condition on \(\Psi\) could be used to show that the map \(\Phi \) is transverse to \(N \). Then, the parametric transversality theorem (Theorem 6.35 in \cite{lee2013smooth}) would imply that for almost all \((v_1,\dots, v_k ) \in V^k \), the resulting map
\begin{equation}
    \sigma: M \to \prod_{M, i=1}^k TM 
\end{equation}
defined by \(\sigma(x) = \Phi((v_1,\dots, v_k ), x) \) is transverse to \(N \). Noting that \(N \) has codimension \(1 \), we would conclude that its preimage  \( \sigma^{-1}(N) \subseteq  M \) is a closed submanifold of codimension 1 (hence a set of measure 0). Since the set defined in \cref{eq:non-degen-locus} is \(M \setminus \sigma^{-1}(N) \) this would complete the proof. 

Again, this proof sketch is merely a heuristic as we know that \(N \) is not smooth. However, there exist \emph{stratified} transversality theorems (hence the discussion of a stratification of \(N \)), and replacing our heuristic application of the more well known parametric transversality theorem with one of these (for example the main theorem of \cite{Trotman}, see also references therein) yields a proof strategy.

\begin{proof}[Proof of Lemma \ref{lem:stable-rank-noninc}]
We first reduce to the case where all matrices are square: if not, letting \(p = \max\{l, m, n\}\) we may embed \(A\) and \(B\) in the upper left block of \(p\times p \) matrices, and direct calculation shows
\begin{equation}
    \begin{pmatrix}
    A & 0\\
    0&0\\
    \end{pmatrix}
    \begin{pmatrix}
    B & 0\\
    0&0\\
    \end{pmatrix}
    = 
    \begin{pmatrix}
    AB & 0\\
    0&0\\
    \end{pmatrix}.
\end{equation}
Moreover, such padding by zeros does not alter singular values, hence leaves all stable ranks in sight unchanged. 

From now on we assume \(A\), \(B\) and  thus \(AB\) are \(n\times n \) matrices. Let \(s(A) = (s(A)_1, \dots, s(A)_n) \) be the \(n\)-dimensional vector of sorted singular values of \(A\) (and similarly for \(B\) and \(AB\)).
By Theorem 
IV.2.5 in \cite{bhatiaMatrixAnalysis1996}, 
\begin{equation}
   s(A B)^2 <_w s(A)^2 s(B)^2
\end{equation}
where \(<_w\) denotes weak submajorization, and the product on the right hand side is coordinatewise (a.k.a. Hadamard) multiplication. From this and the definition of weak submajorization in Section II.1 \cite{bhatiaMatrixAnalysis1996}) it follows that 
\begin{equation}
\label{eq:holder}
    \begin{split}
        &\sum_{i=1}^n s(A B)_i^2 \leq \sum_{i=1}^n s(A)_i^2 s(B)_i^2 \leq s(A)_1^2 \sum_{i=1}^n s(B)_i^2 \\
        & = (s(A)_1 s(B)_1)^2 \frac{1}{s(B)_1^2} \sum_{i=1}^n s(B)_i^2 
    \end{split}
\end{equation}
where the first inequality follows from the definition of weak submajorization in Section II.1 \cite{bhatiaMatrixAnalysis1996}) and the second inequality uses the fact that \( s(A)_1 \geq \dots \geq s(A)_n \).\footnote{Alternatively, this is the \(p=\infty\), \(q=1\) H\"older's inequality.} We now divide both sides of \cref{eq:holder} by \(s(A B)_1^2\) to obtain 
\begin{equation}
    \begin{split}
        &r(AB) = \frac{1}{s(A B)_1^2}\sum_i s(A B)_i^2\\
        &\leq \frac{(s(A)_1 s(B)_1)^2}{s(A B)_1^2} \frac{1}{s(B)_1^2} \sum_i s(B)_i^2 \\
        &=\Big(\frac{\lVert A \rVert_{\mathrm{spec}} \lVert B \rVert_{\mathrm{spec}}}{\lVert AB \rVert_{\mathrm{spec}}}\Big)^2 r(B).
    \end{split}
\end{equation}
By a symmetric argument,\footnote{For example, use the facts that \((AB)^T = B^T A^T\) and transposing doesn't change singular values.}
\begin{equation}
    r(AB) \leq \Big(\frac{\lVert A \rVert_{\mathrm{spec}} \lVert B \rVert_{\mathrm{spec}}}{\lVert AB \rVert_{\mathrm{spec}}}\Big)^2 r(A).
\end{equation}
\end{proof}

\begin{proof}[Proof of \cref{prop:orbit-stab}]
This follows directly from Theorem 3.26 in \cite{kirillov}, which shows that the Lie algebra of \(G_x \) is the kernel of the natural map \(\mathfrak{g} \to T_x M \). Hence if this natural map is injective, \(G_x \) has 0-dimensional Lie algebra and is thus a 0-dimensional closed subgroup of \(G\), i.e. a discrete subgroup. 
\end{proof}

\section{Experimental details}
\label{appendix-experimental-details}

We ran all of our experiments on an Nvidia A100 GPU. The specific hyperparameters that we used in training can be found in Table \ref{table-hyperparameters}.

\subsection{Augmentation frames}
\label{appendix-augmentation-frames}

In Table \ref{tab:augmentation-transformations} we describe the augmentations that we use in the experiments with augmentation frames in this paper, the software library used to implement them, and the parameter settings that we used to approximate the tangent vector corresponding to each.


\subsection{Handling edge effects in image translation and rotation}

Because images have finite support, pixels with zero value appear at the edges when the image is rotated by $\theta$ degrees (where $\theta \neq 0^\circ$, $90^\circ$, $180^\circ$, or $270^\circ$). Thus, these augmentations violate our goal of only using augmentations that produce naturalistic images. To handle this situation in practice, we increase the size of the image so that we can crop out fractional numbers of pixels (after rotation). In detail we:
\begin{enumerate}
    \item resize the image to $\times8$ its original size,
    \item rotate the image by $5$ degrees,
    \item crop out empty pixels at the corners by removing $20$ pixels around the border,
    \item resize the image back to its original size.
\end{enumerate}

\begin{figure*}[h]
\begin{center}
\includegraphics[width=.19\linewidth]{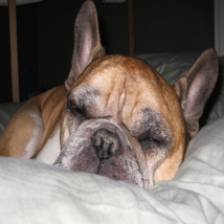}
\includegraphics[width=.19\linewidth]{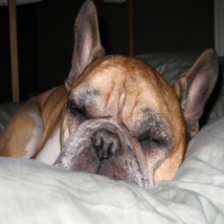}
\includegraphics[width=.19\linewidth]{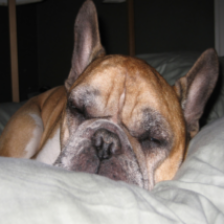}
\includegraphics[width=.19\linewidth]{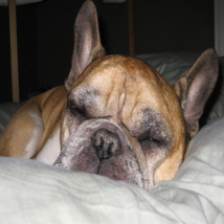}
\includegraphics[width=.19\linewidth]{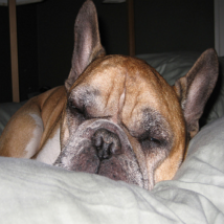}\\
\includegraphics[width=.19\linewidth]{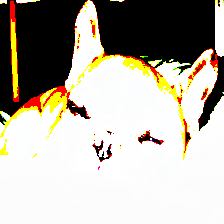}
\includegraphics[width=.19\linewidth]{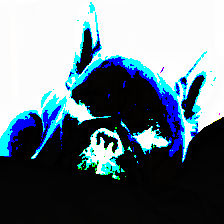}
\includegraphics[width=.19\linewidth]{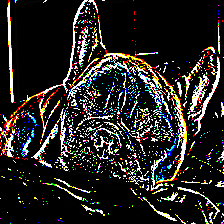}
\includegraphics[width=.19\linewidth]{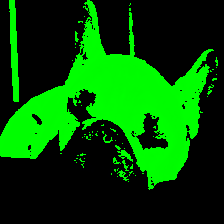}
\includegraphics[width=.19\linewidth]{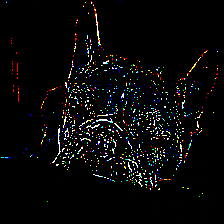}\\
\includegraphics[width=.19\linewidth]{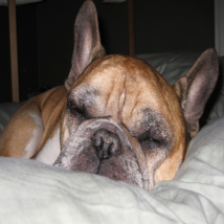}
\includegraphics[width=.19\linewidth]{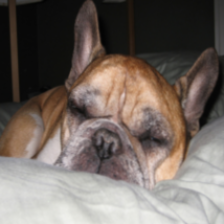}
\includegraphics[width=.19\linewidth]{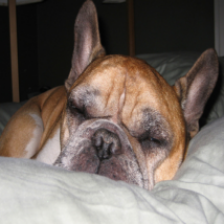}
\includegraphics[width=.19\linewidth]{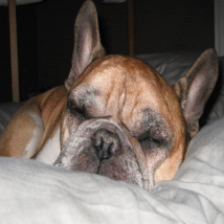}\\
\includegraphics[width=.19\linewidth]{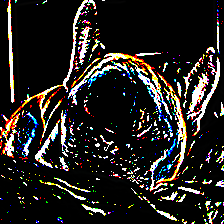}
\includegraphics[width=.19\linewidth]{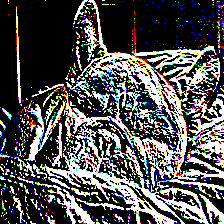}
\includegraphics[width=.19\linewidth]{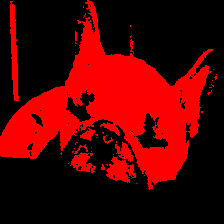}
\includegraphics[width=.19\linewidth]{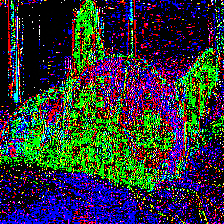}

\end{center}
\caption{\textbf{(First and third rows)} A subset of the sample perturbations, from left to right and top to bottom: brightness, contrast, crop with bilinear interpolation, hue, sharpness, crop with linear interpolation, rotation, saturation, jpeg compression. \textbf{(Second and fourth rows)} The difference between the original and perturbed images (above). \label{fig-example-augmentations}}
\end{figure*}

\begin{figure*}[h]
\begin{center}
\includegraphics[width=.19\linewidth]{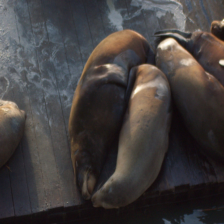}
\includegraphics[width=.19\linewidth]{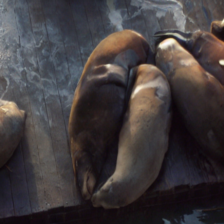}
\includegraphics[width=.19\linewidth]{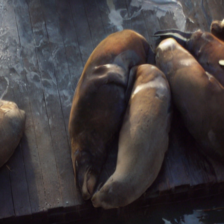}
\includegraphics[width=.19\linewidth]{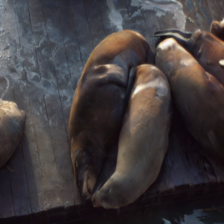}\\
\hspace{29mm}\includegraphics[width=.19\linewidth]{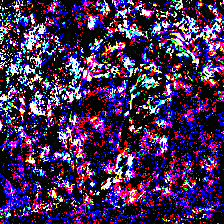}
\includegraphics[width=.19\linewidth]{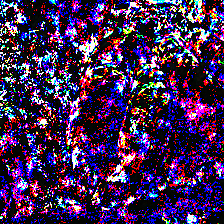}
\includegraphics[width=.19\linewidth]{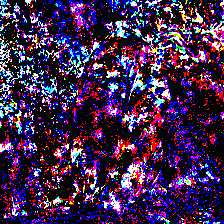}

\end{center}
\caption{\textbf{(First row)} The original ImageNet image and perturbations of this image sampled using the Boomerang method and stable diffusion. \textbf{(Second row)} The difference between the original image and each perturbation. \label{fig-example-sd}}
\end{figure*}

\subsection{Hidden Layers Used for Each Architecture}

A significant amount of this work depends on querying the hidden representations of different models. In Tables \ref{table-vit-layers}-\ref{table-alexnet-layers} we list the hidden layers used in each model type in the experiments. 

\begin{table*}[th]
\caption{Layers for torchvision ViT Base 16. \label{table-vit-layers}}
\begin{center}
\begin{tabular}{lr}
\toprule
\textbf{Layer name} & \textbf{Layer number} 
 \\ \midrule 
\codeword{encoder.layers.encoder_layer_1.mlp} & 1\\
\codeword{encoder.layers.encoder_layer_3.mlp} & 2\\
\codeword{encoder.layers.encoder_layer_5.mlp} & 3\\
\codeword{encoder.layers.encoder_layer_7.mlp} & 4\\
\codeword{encoder.layers.encoder_layer_9.mlp} & 5\\
\codeword{encoder.layers.encoder_layer_11.mlp} & 6\\
\codeword{getitem_5} & 7\\
\end{tabular}
\end{center}
\end{table*}

\begin{table*}[th]
\caption{Layers for torchvision Swin T. \label{table-swin-layers}}
\begin{center}
\begin{tabular}{lr}
\toprule
\textbf{Layer name} & \textbf{Layer number} 
 \\ \midrule 
\codeword{features.1.0.mlp} & 1 \\
\codeword{features.3.0.mlp} & 2 \\
\codeword{features.5.0.mlp} & 3 \\
\codeword{features.5.2.mlp} & 4 \\ 
\codeword{features.5.4.mlp} & 5 \\ 
\codeword{features.7.0.mlp} & 6 \\ 
\codeword{features.7.1.mlp} & 7 \\
\codeword{flatten} & 8\\
\end{tabular}
\end{center}
\end{table*}

\begin{table*}[th]
\caption{Layers for torchvision ResNeXT50 32$\times$4d. \label{table-resnext50-layers}}
\begin{center}
\begin{tabular}{lr}
\toprule
\textbf{Layer name} & \textbf{Layer number} 
 \\ \midrule 
\codeword{layer1.0.add} & 1 \\
\codeword{layer2.0.add} & 2 \\
\codeword{layer3.0.add} & 3 \\
\codeword{layer4.0.add} & 4 \\
\codeword{flatten} & 5 \\
\end{tabular}
\end{center}
\end{table*}

\begin{table*}[th]
\caption{Layers for torchvision MobileNetV3 (small). \label{table-mobilenetv3-layers}}
\begin{center}
\begin{tabular}{lr}
\toprule
\textbf{Layer name} & \textbf{Layer number} 
 \\ \midrule 
\codeword{features.1.block.2} & 1\\
\codeword{features.3.add} & 2\\
\codeword{features.5.add} & 3\\
\codeword{features.7.block.3} & 4\\
\codeword{features.9.block.3} & 5\\
\codeword{features.11.add} & 6\\
\codeword{flatten} & 7\\
\codeword{classifier.0} & 8\\
\codeword{classifier.1} & 9\\
\codeword{classifier.2} & 10\\
\end{tabular}
\end{center}
\end{table*}

\begin{table*}[th]
\caption{Layers for torchvision InceptionV3. \label{table-inceptionv3-layers}}
\begin{center}
\begin{tabular}{lr}
\toprule
\textbf{Layer name} & \textbf{Layer number} 
 \\ \midrule 
\codeword{Conv2d_1a_3x3.relu} & 1\\
\codeword{Conv2d_2b_3x3.relu} & 2\\
\codeword{Conv2d_4a_3x3.relu} & 3\\
\codeword{Mixed_5b.cat} & 4\\
\codeword{Mixed_5d.cat} & 5\\
\codeword{Mixed_6a.cat} & 6\\
\codeword{Mixed_6c.cat} & 7\\
\codeword{Mixed_6e.cat} & 8\\
\codeword{Mixed_7a.cat} & 9\\
\codeword{Mixed_7b.cat_2} & 10\\
\codeword{flatten} & 11
\end{tabular}
\end{center}
\end{table*}

\begin{table*}[th]
\caption{Layers for torchvision DenseNet121. \label{table-densenet121-layers}}
\begin{center}
\begin{tabular}{lr}
\toprule
\textbf{Layer name} & \textbf{Layer number} 
 \\ \midrule 
\codeword{features.denseblock1.denselayer2.cat} & 1\\
\codeword{features.denseblock2.denselayer1.cat} & 2\\
\codeword{features.denseblock2.denselayer7.cat} & 3\\
\codeword{features.denseblock2.cat} & 4\\
\codeword{features.denseblock3.denselayer6.cat} & 5\\
\codeword{features.denseblock3.denselayer12.cat} & 6\\
\codeword{features.denseblock3.denselayer18.cat} & 7\\
\codeword{features.denseblock3.cat} & 8\\
\codeword{features.denseblock4.denselayer6.cat} & 9\\
\codeword{features.denseblock4.denselayer12.cat} & 10\\
\codeword{features.denseblock4.cat} & 11\\
\codeword{flatten} & 12\\
\end{tabular}
\end{center}
\end{table*}

\begin{table*}[th]
\caption{Layers for timm resnetv2 101$\times$1 (BiT). \label{table-bit-layers}}
\begin{center}
\begin{tabular}{lr}
\toprule
\textbf{Layer name} & \textbf{Layer number} 
 \\ \midrule 
\codeword{stages.0.blocks.0.add} & 1\\
\codeword{stages.1.blocks.0.add} & 2\\
\codeword{stages.2.blocks.0.add} & 3\\
\codeword{stages.2.blocks.4.add} & 4\\
\codeword{stages.2.blocks.8.add} & 5\\
\codeword{stages.2.blocks.12.add} & 6\\
\codeword{stages.2.blocks.16.add} & 7\\
\codeword{stages.2.blocks.20.add} & 8\\
\codeword{stages.3.blocks.2.add} & 9\\
\codeword{head.global_pool.flatten} & 10\\
\end{tabular}
\end{center}
\end{table*}

\begin{table*}[th]
\caption{Layers for torchvision ConvNeXT small. \label{table-convnext-layers}}
\begin{center}
\begin{tabular}{lr}
\toprule
\textbf{Layer name} & \textbf{Layer number} 
 \\ \midrule 
\codeword{features.1.0.block.0} & 1 \\
\codeword{features.3.0.block.0} & 2 \\
\codeword{features.5.0.block.0} & 3 \\ 
\codeword{features.5.8.block.0} & 4 \\
\codeword{features.5.16.block.0} & 5 \\
\codeword{features.7.1.block.0} & 6 \\
\codeword{classifier.1} & 7 \\
\end{tabular}
\end{center}
\end{table*}

\begin{table*}[th]
\caption{Layers for torchvision AlexNet. \label{table-alexnet-layers}}
\begin{center}
\begin{tabular}{lr}
\toprule
\textbf{Layer name} & \textbf{Layer number} 
 \\ \midrule 
\codeword{features.1} & 1\\
\codeword{features.4} & 2\\
\codeword{features.7} & 3\\
\codeword{features.9} & 4\\
\codeword{features.11} & 5\\
\codeword{classifier.2} & 6\\
\codeword{classifier.5} & 7\\
\codeword{flatten} & 8\\
\end{tabular}
\end{center}
\end{table*}

\begin{table*}[th]
\caption{Layers for torchvision ResNet50. \label{table-resnet50-layers}}
\begin{center}
\begin{tabular}{lr}
\toprule
\textbf{Layer name} & \textbf{Layer number} 
 \\ \midrule 
\codeword{layer1.2.relu_2} & 1\\
\codeword{layer2.3.relu_2} & 2\\
\codeword{layer3.5.relu_2} & 3\\
\codeword{layer4.2.relu_2} & 4\\
\codeword{flatten} & 5\\
\end{tabular}
\end{center}
\end{table*}

\begin{table*}[th]
\caption{Layers for torchvision ResNet18. \label{table-resnet18-layers}}
\begin{center}
\begin{tabular}{lr}
\toprule
\textbf{Layer name} & \textbf{Layer number} 
 \\ \midrule 
\codeword{layer1.1.relu_1} & 1\\
\codeword{layer2.1.relu_1} & 2\\
\codeword{layer3.1.relu_1} & 3\\
\codeword{layer4.0.relu_1} & 4\\
\codeword{flatten} & 5\\
\end{tabular}
\end{center}
\end{table*}

\begin{table}[th]
\caption{Hyperparameters used when training the ResNet18 on ImageNet. \label{table-hyperparameters}}
\begin{center}
\begin{tabular}{rrrr}
\toprule
Training hyperparameters &  \\
\midrule 
Optimzer & SGD \\
Learning rate  & 0.5  \\
Learning rate scheduler  & cyclic  \\
Learning rate peak epoch & 2  \\
Momentum & 0.9  \\
Batch size       & 1024   \\
Epochs & 16\\
Weight decay & $5\mathrm{e}{-5}$\\
Label smoothing & 0.1\\
BlurPool? & Yes\\
Pretrained? & No\\
\end{tabular}
\end{center}
\end{table}

\end{document}